\documentclass{article}

\usepackage{microtype}
\usepackage{graphicx}
\usepackage{subfigure}
\usepackage{booktabs} 
\usepackage{stackengine}
\usepackage[table,xcdraw]{xcolor}

\usepackage{hyperref}

\usepackage[accepted]{icml2023}

\usepackage{amsmath}
\usepackage{amssymb}
\usepackage{mathtools}
\usepackage{amsthm}

\usepackage[capitalize,noabbrev]{cleveref}

\theoremstyle{plain}
\newtheorem{theorem}{Theorem}[section]

\theoremstyle{definition}

\theoremstyle{remark}
\newtheorem{remark}[theorem]{Remark}

\usepackage[textsize=tiny]{todonotes}

\icmltitlerunning{Image generation with shortest path diffusion}

\begin{document}

\twocolumn[
\icmltitle{Image generation with shortest path diffusion}



\icmlsetsymbol{equal}{*}

\begin{icmlauthorlist}
\icmlauthor{Ayan Das}{equal,comp}
\icmlauthor{Stathi Fotiadis}{equal,comp,imp}
\icmlauthor{Anil Batra}{comp,edin}
\icmlauthor{Farhang Nabiei}{comp}
\icmlauthor{FengTing Liao}{comp}
\icmlauthor{Sattar Vakili}{comp}
\icmlauthor{Da-shan Shiu}{comp}
\icmlauthor{Alberto Bernacchia}{comp}
\end{icmlauthorlist}

\icmlaffiliation{comp}{MediaTek Research, Cambourne, UK}
\icmlaffiliation{imp}{Department of Bioengineering, Imperial College London, London, UK}
\icmlaffiliation{edin}{School of Informatics, University of Edinburgh, Edinburgh, UK}

\icmlcorrespondingauthor{Ayan Das}{ayan.das@mtkresearch.com}


\vskip 0.3in
]



%
\printAffiliationsAndNotice{\icmlEqualContribution} 

\begin{abstract}
The field of image generation has made significant progress thanks to the introduction of Diffusion Models, which learn to progressively reverse a given image corruption.
Recently, a few studies introduced alternative ways of corrupting images in Diffusion Models, with an emphasis on blurring.
However, these studies are purely empirical and it remains unclear what is the optimal procedure for corrupting an image.
In this work, we hypothesize that the optimal procedure minimizes the length of the path taken when corrupting an image towards a given final state.
We propose the Fisher metric for the path length, measured in the space of probability distributions.
We compute the shortest path according to this metric, and we show that it corresponds to a combination of image sharpening, rather than blurring, and noise deblurring.
While the corruption was chosen arbitrarily in previous work, our Shortest Path Diffusion (SPD) determines uniquely the entire spatiotemporal structure of the corruption.
We show that SPD improves on strong baselines without any hyperparameter tuning, and outperforms all previous Diffusion Models based on image blurring.
Furthermore, any small deviation from the shortest path leads to worse performance, suggesting that SPD provides the optimal procedure to corrupt images.
Our work sheds new light on observations made in recent works and provides a new approach to improve diffusion models on images and other types of data.
\end{abstract}

\begin{figure}[!th]
  \centering
     \includegraphics[width=\linewidth]{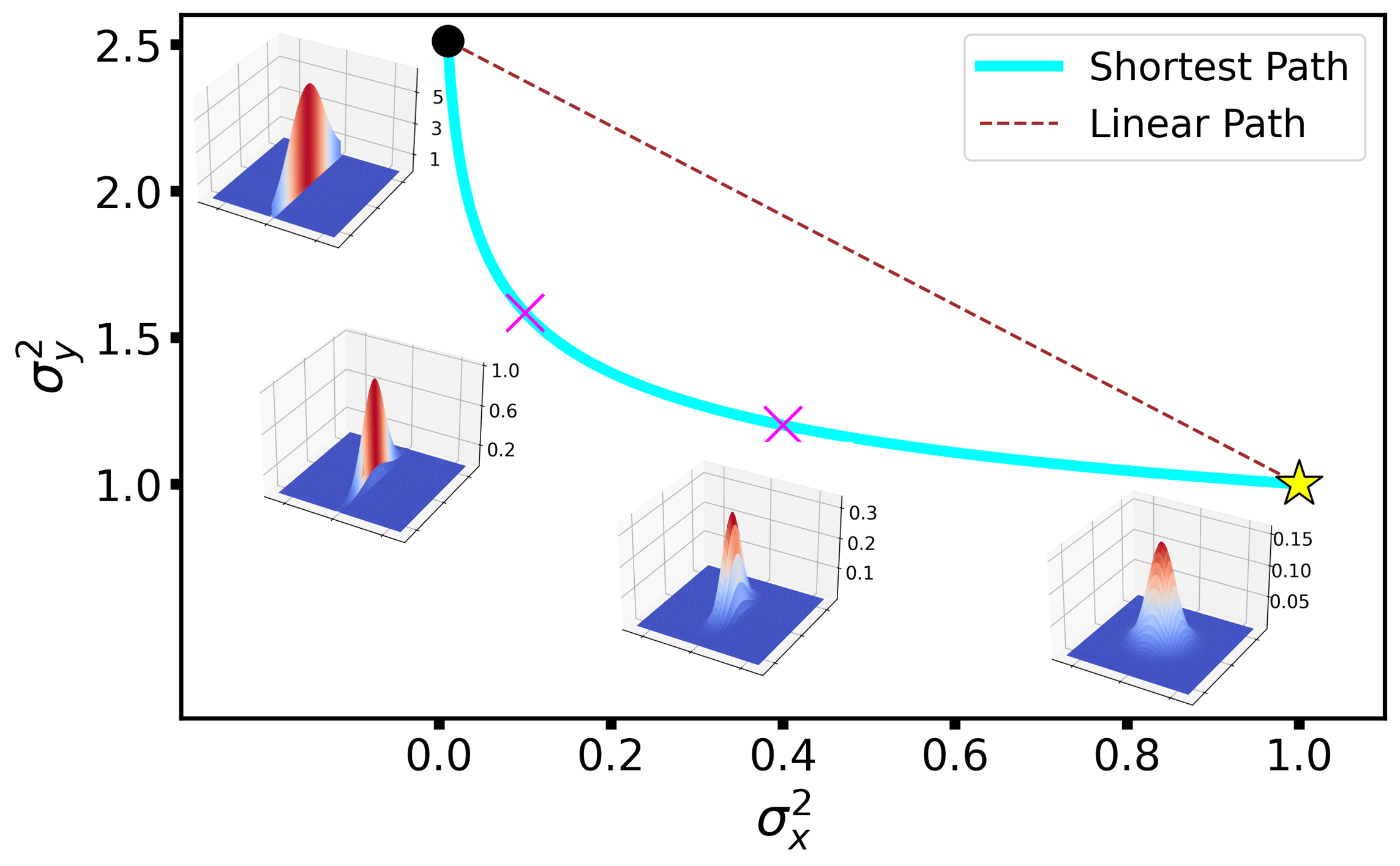}
   \includegraphics[width=\linewidth]{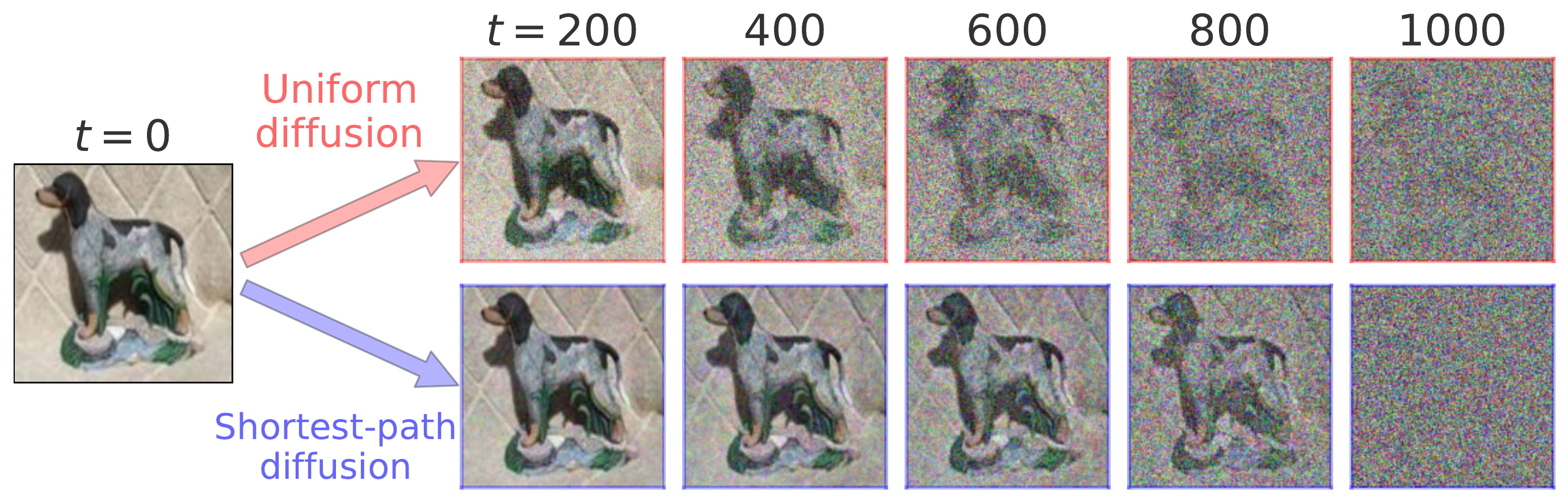}
   \caption{\textbf{Illustration of shortest path}. 
   \textbf{Top}: Transformation of a bivariate Gaussian distribution, parameterized by the variance along two orthogonal directions, $\sigma_x^2$ and $\sigma_y^2$. 
   The initial distribution has $\sigma_x^2\ll \sigma_y^2$ (black circle), while the final distribution has $\sigma_x^2= \sigma_y^2$ (yellow star, isotropic noise). 
   According to the Fisher metric, the shortest path between the two distributions is \emph{not} the linear path (dashed line), instead is given by the curved path (cyan curve), in which $\sigma_y^2$ decreases first, and $\sigma_x^2$ increases later. 
   \textbf{Bottom}: Comparison of shortest path and uniform noising for image corruption. In uniform noising, the original image dissipates while uniform noise appears. Instead, the shortest path corresponds to image sharpening and noise deblurring. Lower frequencies of the image dissipate before higher frequencies. Similarly, noise appears at lower frequencies first and higher frequencies later. Figure \ref{fig:speed} shows how signal and noise change in time for different frequencies.}
   \label{fig:path}
\end{figure}

\section{Introduction}
\label{sec:intro}

The field of image generation has seen rapid progress since the introduction of algorithms based on deep learning \cite{9555209}.
A common approach is using a deep neural network to map input noise into an output image, a fast process that requires a single forward pass. 
These methods include Generative Adversarial Networks \cite{goodfellow2020generative, karras2020analyzing}, which provide good image quality, Variational Autoencoders \cite{kingma2014auto, child2020very} and Normalizing Flows \cite{dinh2016density, chen2019residual}, which provide a rich diversity of sampled images.
Other approaches based on deep learning include Autoregressive Models \cite{van2016pixel, child2019generating}, which generate one pixel (or patch) at a time.

Diffusion models \cite{sohl2015deep, song2019generative, ho2020denoising} represent an alternative family of algorithms outperforming previous approaches both in terms of quality \cite{dhariwal2021diffusion} and diversity \cite{kingma2021variational}.
Similar to previous approaches, diffusion models transform input noise into an output image.
However, instead of generating an image by a single forward pass through a neural network, diffusion models use multiple steps of denoising, which require multiple forward passes.
This iterative procedure allows refining an image to unprecedented quality.
The combination of diffusion and language models led to impressive progress in text-to-image generation \cite{saharia2022photorealistic, nichol2021glide}.

Recent work questioned the procedure for corrupting images in diffusion models.
Early work proposed using noise \cite{sohl2015deep, song2019generative, ho2020denoising}, but recent studies explored alternative procedures, with a strong focus on image blurring \cite{rissanen2022generative, lee2022progressive, bansal2022cold, daras2022soft, hoogeboom2022blurring}. In all previous work, corruptions are chosen arbitrarily and it remains unclear what is the optimal procedure for corruption. In this work, we provide a candidate for the optimal procedure. We reason that the main sources of errors in diffusion models are the approximations made in reversing the corruption. Therefore, the optimal corruption procedure would be one that minimizes those errors.  

A procedure for corrupting images in a diffusion model is equivalent to a transformation of the data distribution into another probability distribution, which is often taken to be an isotropic Gaussian. For a given parameterization of the probability distributions, this transformation can be visualized by a path in the space of parameters, from the parameter values of the data distribution to those of an isotropic Gaussian (see figure \ref{fig:path}).
Any given procedure for corruption corresponds to a path in the space of distributions.
We hypothesize that the optimal procedure for corruption corresponds to the shortest path.
The intuition is that any error made in approximating the true reversal of the corruption, accumulated along the path, would be smaller if the path is shorter. 

To compute the path length, a metric in the space of distributions needs to be defined.
We choose the Fisher Information Matrix as metric, because it bounds the precision of maximum likelihood estimation \cite{amari2016information}, and Diffusion Models are trained by using a bound on the likelihood as the loss function \cite{sohl2015deep, ho2020denoising}.
Furthermore, the Fisher metric is reparameterization invariant, namely any length computed by the metric does not depend on the choice made for parameterizing the distributions \cite{amari2016information}.
We contribute the following:

\begin{itemize}
    \item We compute analytically the shortest path between Gaussian distributions, and we propose an approximation for the non-Gaussian case (e.g. image data).
    \item We show that the shortest path corresponds to a combination of image sharpening and noise deblurring. We provide the exact formula for the specific corruption prescribed by the shortest path.
    \item We test our Shortest Path Diffusion (SPD) on CIFAR10. We show that any small departure from the shortest path results in worse performance, and SPD outperforms all methods based on image blurring. Our results suggest that SPD provides the optimal corruption.
    \item We also test SPD on ImagNet $64\times64$, on the task of \emph{unconditional} generation, and we show that SPD improves on strong baselines without any hyperparameter tuning. 
\end{itemize}

\section{Related work}
\label{sec:related}

\subsection{Diffusion Models}

Diffusion probabilistic models were introduced by \citet{sohl2015deep} following a line of research on Markov chain-based generative models \cite{bengio2014deep, salimans2015markov}. 
The work of \citet{song2019generative, song2020improved} proposed an algorithm for image generation based on learning the score of the data distribution.
The work of \citet{ho2020denoising} pointed out the equivalence between diffusion and score-based generative models, and showed the capability of these models in generating high-quality images. The two approaches were unified by the framework of stochastic differential equations \cite{song2020score}. The work of \citet{song2020denoising} introduced non-Markovian diffusion models, allowing deterministic and faster sampling.

\subsection{Variance schedule}

One of the first avenues to improve the performance of Diffusion Models was adjusting the temporal properties of the corruption process, also known as variance schedule, while its spatial properties remained  fixed to isotropic noise.
Early work used a linear or exponential increase of variance.
The work of \citet{iddpm} introduced a cosine function of time, which enabled better quality of generated images.
The work of \citet{kingma2021variational} showed that learning the variance schedule improves performance on image density estimation benchmarks.
In all previous work, variance scheduling was chosen arbitrarily.
In our work instead, the entire spatio-temporal properties of the corruption are determined by the shortest path.

\subsection{Image blurring}

Most of the previous work on Diffusion Models fixed the spatial corruption to isotropic noise.
However, a few recent studies introduced alternative image corruptions, in which different frequencies are degraded at different times \cite{rissanen2022generative, lee2022progressive, bansal2022cold, daras2022soft, hoogeboom2022blurring}.
When higher frequencies are degraded first, the observed effect is image blurring. The work of \cite{rissanen2022generative} proposes simulating the heat equation, given the equivalence of heat dissipation and Gaussian blurring.
The work of \citet{hoogeboom2022blurring} combines heat dissipation and additive noise, formalizing it as a diffusion process with anisotropic noise.
The work of \citet{lee2022progressive} introduced Gaussian blur with monotonically increasing power while following the variance schedule of \citet{ho2020denoising}. 
The works of \citet{bansal2022cold} and \citet{daras2022soft} extended Diffusion Models to a wide variety of corruptions, including Gaussian blur. 
In all these studies, the choice of corruption is arbitrary.
The justification for using image blurring is that low-frequency features are more important for human perception of images.
In our work instead, we show that the shortest path corresponds to image sharpening and noise deblurring.

\subsection{Reverse process}

Diffusion models may differ in the parametrization of the reverse process.
A neural network may be trained to either predict the noise or the original image. 
The work of \citet{ho2020denoising} found that predicting noise results in better performance. 
However, recent results show that using a combination of the two parametrizations may improve the quality of generated images \cite{benny2022dynamic}. 
The work of \citet{salimans2022progressive} used progressive distillation to skip iterations in the reverse process, and showed that predicting the original image instead of noise achieves the best sampling performance after distillation. 
We use noise prediction in our work, similar to \citet{ho2020denoising}, but other parameterizations may be implemented as well.

The work of \citet{ma2022accelerating} introduced a matrix pre-conditioning method to accelarete the reverse process in score-based models. 
The work of \citet{bao2022analytic, bao2022estimating} proposed learning the optimal covariance of the reverse process and showed significant improvements on image quality.
Recently, \citet{lu2022dpmsolver} simplified the expression of the reverse diffusion and significantly improved the speed of generation and quality of images.
These features could be added to SPD in the future and are likely to provide further improvements.

\subsection{Other improvements}

The work of \citet{guth2022wavelet} uses orthogonal wavelets to decompose the images and generates samples in the space of wavelets. In our work, we use a Fourier basis instead of a wavelet basis, and our noising procedure is non-uniform. The work of \citet{lee2022priorgrad} replaces the isotropic Gaussian prior with a distribution of mean and covariance computed on the data. Similar to our work, this approach may shorten the trajectory between the dataset and the prior, but it may not correspond to the shortest path.
The work of \citet{khrulkov2022understanding} shows that diffusion models implement the optimal transport from the data to the target distribution. However, our work is concerned with optimality of the trajectory, rather than of the mapping between the initial and final state.

Other improvements of Diffusion Models that are orthogonal to our work include:
The work of \citet{dockhorn2021score}, which increased the sampling speed by augmenting the diffusion process with auxiliary variables. 
The works of \citet{vahdat2021score} and \citet{jing2022subspace}, which propose diffusing in a latent space, improving the generation quality and the computational costs. The work of \citet{watson2021learning}, which designed a differentiable parametric sampler that can be optimized for fast data generation. All of these can be in principle added to our algorithm and are expected to provide further improvements.

\section{Shortest Path Diffusion}
\label{sec:theory}

In this section, we compute analytically the shortest path for Gaussian distributions, and propose an application to non-Gaussian case, in particular to natural images. 
We provide the algorithm of Shortest Path Diffusion and discuss its complexity.
Details of derivations and proofs are provided in the appendix.

\subsection{Shortest path for Gaussian distributions}
\label{shortgauss}

We consider the simple case of an image $\mathbf{x}$ distributed according to a Gaussian distribution with zero mean and covariance matrix $\Sigma$,
\begin{equation}
\mathbf{x}\sim\mathcal{N}\left(\mathbf{0},\Sigma\right)
\end{equation}
The vector $\mathbf{x}$ concatenates all pixels of an image, and the matrix $\Sigma$ includes the covariances of all pairs of pixels.
For example, a $32\times 32$ image has $1024$ pixels, thus $\mathbf{x}$ is a vector of $1024$ elements and $\Sigma$ is a $1024\times1024$ matrix.
The probability distribution is completely described by $\Sigma$, and any transformation from an initial distribution at time $t=0$ to a final distribution at time $t=T$ is described by the temporal change in covariance $\Sigma_t$.

The sequence $\Sigma_t$ describes a path in the space of probability distributions.
Given the initial $\Sigma_0$ and the final $\Sigma_T$, what is the shortest path $\Sigma_t$?
Here we choose to measure path lengths using the Fisher metric, because it provides a bound on the precision of maximum likelihood estimation of probability distributions via the Cramer-Rao theorem \cite{amari2016information}.
The reason is that Diffusion Models are trained by the Evidence Lower Bound (ELBO), which is a bound on the likelihood \cite{sohl2015deep, ho2020denoising}.
We also highlight that the Fisher metric is invariant for reparameterization of probability distributions \cite{amari2016information}, therefore any measured path length (and thus the shortest path) does not depend on the chosen parameterization of the distribution (for example, using the precision instead of the covariance, or any other invertible transformation of the covariance).

\begin{theorem}
\label{th:sp-main}
Given two Gaussian distributions with zero mean and covariance matrix equal to, respectively, $\Sigma_0$ and $\Sigma_1$, where $\Sigma_1$ is non-singular.
Given the Riemannian metric defined by the Fisher information, the shortest path between the two distributions is given by
\begin{equation}
\Sigma_t=\Sigma_1^{1/2}\left(\Sigma_1^{-1/2}\Sigma_0\Sigma_1^{-1/2}\right)^{1-t}\Sigma_1^{1/2}
\end{equation}
where $t\in(0,1)$ measures the relative distance travelled along the path.

\end{theorem}

Proof is provided in appendix \ref{app:fisher proof} (see also \citet{pinele2020fisher}).
For numerical purposes, we implement a discrete-time version of the shortest path, with $t=0,1,2, ..., T$.
Consistent with previous work, we set the final distribution as isotropic Gaussian, thus the final covariance is equal to $\Sigma_T=\mbox{I}$ (identity matrix) for any initial covariance $\Sigma_0$.
In the shortest path, each eigenvalue $\sigma^2$ of the covariance matrix, representing the variance of a given combination of pixels, evolves according to
\begin{equation}
\label{eq:SPvar}
\sigma^2_t=\left(\sigma^2_0\right)^{1-t/T}
\end{equation}
where $\sigma_0^2$ is the initial variance.
Figure \ref{fig:path} (top) shows an example with two variances, where the large one drops first and the smaller one raises later.
We stress that the metric is not Euclidean therefore the linear path in figure \ref{fig:path} is not the shortest path.
The exponential dependence on time in equation \ref{eq:SPvar} implies that the rate of change depends on the initial variance: small $\sigma_0^2$ change slowly, while large $\sigma_0^2$ change faster.

In matrix form, the shortest path corresponds to the following covariance schedule
\begin{equation}
\label{eq:SP}
\Sigma_t=FD^{1-t/T}F^\dagger
\end{equation}
where $F$ is the matrix of orthogonal eigenvectors of $\Sigma_0$ and $D$ is the diagonal matrix of positive eigenvalues of $\Sigma_0$.
Here $\dagger$ denotes the operations of matrix transpose and complex conjugation, and $D^{1-t/T}$ is the diagonal matrix where each element is raised to the power of $1-t/T$.

Equation \ref{eq:SP} describes the shortest path between a Gaussian distribution with covariance $\Sigma_0$ and an isotropic Gaussian.
In the context of diffusion models, this is a corruption procedure that starts from the data distribution, which has a rich structure described by the covariance $\Sigma_0$, and terminates with pure noise (isotropic Gaussian).
This corruption procedure is implemented by a forward process that corrupts individual images sampled from the data distribution \cite{ho2020denoising}.
In the next section, we derive a corruption procedure implementing the shortest path.

\subsection{Image corruption} \label{image-corruption}

The following theorem provides a data corruption procedure implementing the shortest path of equation \ref{eq:SP}.

\begin{theorem}
Given a random vector $\mathbf{x}_0$ of zero mean and covariance $\Sigma_0$, and another random vector ${\boldsymbol \epsilon}_t$ of zero mean and covariance $\mbox{\emph{I}}$ (isotropic), where $\mathbf{x}_0$ and ${\boldsymbol \epsilon}_t$ are uncorrelated.
Assume the matrix $(\mbox{I}-\Sigma_0)$ is invertible.
Define the matrix $\Phi_t = (\mbox{I} -\Sigma_0^{1-t})(\mbox{I}-\Sigma_0)^{-1}$ and note that, for $t\in(0,1)$, $\Phi_t$ and $(\mbox{\emph{I}}-\Phi_t)$ are positive definite and, respectively, monotonically decreasing and increasing functions of $\Sigma_0$.
Then, the corrupted vector $\mathbf{x}_t$ defined by
\begin{equation}
\label{eq:FP}
    \mathbf{x}_t = \Phi_t^{\frac{1}{2}} \mathbf{x}_0 + (\mbox{I}-\Phi_t)^{\frac{1}{2}} {\boldsymbol \epsilon}_t
\end{equation}
has zero mean and covariance equal to $\Sigma_0^{1-t}$.

\end{theorem}

Proof is provided in appendix \ref{app:forward process}.
Therefore, shortest path is implemented by corrupting images from $\mathbf{x}_0$ to $\mathbf{x}_T$ according to equation \ref{eq:FP}, where ${\boldsymbol \epsilon}_t$ is sampled from an isotropic Gaussian.
$\Phi_t$ is a matrix, thus the image $\mathbf{x}_0$ is corrupted by a linear transformation (instead of a simple rescaling \cite{ho2020denoising}).
The noise ${\boldsymbol \epsilon}_t$ is also linearly transformed.
The linear transform $\Phi_t$ is equal to
\begin{equation}
\label{eq:phi}
    \Phi_t = (\mbox{I}-\Sigma_0^{1-t/T})(\mbox{I}-\Sigma_0)^{-1}.
\end{equation}
%

Equation \ref{eq:FP} is similar to the forward process of a few recent studies \cite{rissanen2022generative, lee2022progressive, bansal2022cold, daras2022soft, hoogeboom2022blurring}.
However, those studies picked an arbitrary form of the matrix $\Phi_t$ and tried to optimize it empirically.
Instead, our work provides an optimal form, given by equation \ref{eq:phi}.

\subsection{Application to real images}
\label{sec:natural}

The distribution of real images is not Gaussian, therefore the shortest path of section \ref{shortgauss} does not apply.
However, its covariance matrix $\Sigma_0$ has a rich structure describing the second-order statistics of all pairs of pixels.
We propose to approximate the shortest path between the distribution of real images and an isotropic Gaussian by corrupting images with the forward process of equations \ref{eq:FP}, \ref{eq:phi}.
We note that, even if the distribution of real images is not Gaussian, the forward process \ref{eq:FP}, \ref{eq:phi} still implies the covariance schedule \ref{eq:SP}.
For non-Gaussian distributions, it is unknown whether the shortest path has covariance schedule \ref{eq:SP}, but we hypothesize that this forward process provides a good approximation to the true shortest path.

The application of equations \ref{eq:FP}, \ref{eq:phi} to real images requires computing the covariance matrix $\Sigma_0$ of their distribution.
However, this computation may be expensive, for example $1024\times1024$ images have a $1049600\times1049600$ covariance matrix.
Fortunately, the form of the covariance matrix of translation invariant distributions is known to be equal to
\begin{equation}
\label{eq:sigma0}
\Sigma_0=FDF^\dagger
\end{equation}
where $F$ is the 2-dimensional Discrete Fourier Transform (DFT) matrix, and $D$ is a diagonal matrix with the power spectrum of the data.
We work under the assumption that natural images are approximately translation invariant (this may not apply to certain datasets, e.g. centered faces, CelebA).
Therefore, the eigenvectors of $\Sigma_0$ are given by the DFT matrix, and its eigenvalues are given by the power spectrum.

The power spectrum of natural images is also known to decrease with the squared frequency  norm \cite{hyvarinen2009natural}. Thus, we model the power spectrum with the following equation
%
%

%
\begin{equation}
\label{eq:pow}
D_{ii} = \frac{c_1}{|c_2+f_i|^m}
\end{equation}

where $f_i$ is the frequency corresponding to index $i$, and is equal to the norm of the vector of frequencies along the horizontal and vertical axes of an image
\begin{equation}
\label{eq:pow-freq}
f = \sqrt{f_x^2+f_y^2}
\end{equation}
 We set the exponent $m$ equal to $2$ in most of our experiments, following \cite{hyvarinen2009natural}, while we fit the constants $c_1$, $c_2$ on the empirical power spectrum of the dataset. Figure \ref{fig:pow} shows the power spectrum computed for CIFAR10 \cite{cifar10dataset} and ImageNet $64\times 64$ \cite{5206848} datasets.
We show in section \ref{sec:results} (see figure \ref{fig:exponent}) that using any values of $m$ different from $m=2$ results in worse performance, suggesting that the shortest path is the optimal procedure for corrupting images.

\subsection{Algorithm and complexity}

Training of Shortest Path Diffusion is described in algorithm \ref{alg:spd}.
We note that, in general, equation \ref{eq:FP} requires a large amount of memory and compute, due to the quadratic scaling of $\Phi_t$ with the dimension of data $d$ (number of pixels).
However, in this section we provide an implementation that scales linearly in the case of real images, which neither require computation of $F$ nor any $d\times d$ matrix multiplication. 

Similar to the work of \citet{lee2022progressive} and \citet{hoogeboom2022blurring}, we corrupt images in frequency space instead of pixel space.
We denote by $\mathbf{u}_t$ the 2-dimensional DFT of image $\mathbf{x}_t$, equal to
\begin{equation}
\label{eq:DFT}
\mathbf{u}_t=F^\dagger\mathbf{x}_t
\end{equation}
Given the transformed $\mathbf{u}_t$, we can recover the image by just using inverse Fourier transform
\begin{equation}
\label{eq:iDFT}
\mathbf{x}_t=F\mathbf{u}_t
\end{equation}
Note that the complexity of DFT is quasilinear in $d$ (log-linear). 
Application of equations \ref{eq:FP}, \ref{eq:phi} in frequency space is given by
\begin{equation}
\label{eq:FPfreq}
\mathbf{u}_t = \Psi_t^{\frac{1}{2}} \mathbf{u}_0 + (\mbox{I}-\Psi_t)^{\frac{1}{2}}{\boldsymbol \xi}_t
\end{equation}
\begin{equation}
\label{eq:phifreq}
    \Psi_t = (\mbox{I}-D^{1-t/T})(\mbox{I}-D)^{-1}
\end{equation}
where ${\boldsymbol \xi}_t$ is noise in frequency space, ${\boldsymbol \xi}_t=F^\dagger{\boldsymbol \epsilon}_t$.
The matrix $\Psi$ in equation \ref{eq:phifreq} is diagonal, thus equation \ref{eq:FPfreq} can be implemented by element-wise multiplication and does not require any $d\times d$ matrix multiplication.
Fitting the power spectrum of the data also scales linearly (see appendix \ref{app:ps-time}), thus overall complexity of the algorithm is linear in $d$.

Algorithm \ref{alg:spd} implements a batch size equal to one, but it is straightforward to implement it for larger batch sizes.
We use the \emph{simple} loss function defined in \citet{ho2020denoising}, which trains a neural network $g_\theta$ to estimate the mapping $\hat{{\boldsymbol \epsilon}}_t=g_\theta(\mathbf{x_t})$.

\begin{algorithm}[tb]
   \caption{ Shortest Path Diffusion (batch size = $1$)}
   \label{alg:spd}
\begin{algorithmic}
   \STATE {\bfseries Given:} dataset and randomly initialized network $g_\theta$
   \STATE Compute power spectrum of dataset
   \STATE Fit $c_1, c_2$ on power spectrum with model (\ref{eq:pow})
   \STATE Compute optimal filter $\Psi_t$ for all $t=1:T$ (\ref{eq:phifreq})
   \WHILE{not converged}
   \STATE Sample $\mathbf{x}_0$ from dataset and compute its DFT $\mathbf{u}_0$
   \STATE Sample $t$ uniformly in $1:T$
   \STATE Sample noise ${\boldsymbol \epsilon}_t$ and compute its DFT $\boldsymbol{\xi}_t$
   \STATE Compute corrupted $\mathbf{u}_t$ (\ref{eq:FPfreq}) and its inverse DFT $\mathbf{x}_t$
   \STATE One-step optimization of $\theta$ with loss$(g_\theta(\mathbf{x}_t),{\boldsymbol \epsilon}_t)$
   \ENDWHILE
\end{algorithmic}
\end{algorithm}


\subsection{Image generation}

Algorithm \ref{alg:image gen} describes the algorithm for image generation (reverse process).
For generation of images we essentially follow \citet{ho2020denoising}.
However, similar to the forward process, the reverse process also runs in frequency space, as in recent works \cite{lee2022progressive, hoogeboom2022blurring}.
After training a neural network on the mapping
$\hat{{\boldsymbol \epsilon}}_t=g_\theta(\mathbf{x_t})$ by Shortest Path Diffusion, we use it to approximate the reverse process according to
\begin{equation}
\label{eq:phirev}
\begin{split}
&\mathbf{u}_{t-1}=\Psi_t^{-\frac{1}{2}}\Psi_{t-1}^{\frac{1}{2}}\mathbf{u}_t+\sigma_t F^\dagger \mathbf{z}_t\\
&-\Psi_t^{-\frac{1}{2}}\Psi_{t-1}^{\frac{1}{2}}(\mbox{I}-\Psi_t\Psi_{t-1}^{-1})(\mbox{I}-\Psi_t)^{-\frac{1}{2}}F^\dagger g_\theta(\mathbf{x}_t)
\end{split}
\end{equation}
where $\Psi_t$ is diagonal and applies element-wise, $\mathbf{z}_t$ is isotropic Gaussian noise and $\sigma_t$ is chosen depending on $T$ (also diagonal, see section \ref{sec:experiments}).

While the neural network operates in pixel space, we use DFT to compute the transformed estimate and run the reverse process in frequency space.
This allows using previously successful neural network architectures, which are known to operate well in pixel space.


\begin{algorithm}[tb]
   \caption{Image generation (reverse process)}
   \label{alg:image gen}
\begin{algorithmic}
   \STATE {\bfseries Given:} trained neural network $g_\theta$ and optimal filter $\Psi_t$
   \STATE Set noise $\sigma_t$  for all $t=1:T$
   \STATE Set $t=T$
   \STATE Sample $\mathbf{x}_T\sim\mathcal{N}(0,\mbox{I})$ and compute its DFT $\mathbf{u}_T$
   \WHILE{$t>0$}
   \STATE Sample $\mathbf{z}_t\sim\mathcal{N}(0,\mbox{I})$
   \STATE Compute $\mathbf{u}_{t-1}$ (\ref{eq:phirev}) and its inverse DFT $\mathbf{x}_{t-1}$
   \STATE $t=t-1$
   \ENDWHILE
   \STATE {\bfseries Return} $\mathbf{x}_0$
\end{algorithmic}
\end{algorithm}

\def\big{\centering\includegraphics[width=\linewidth]{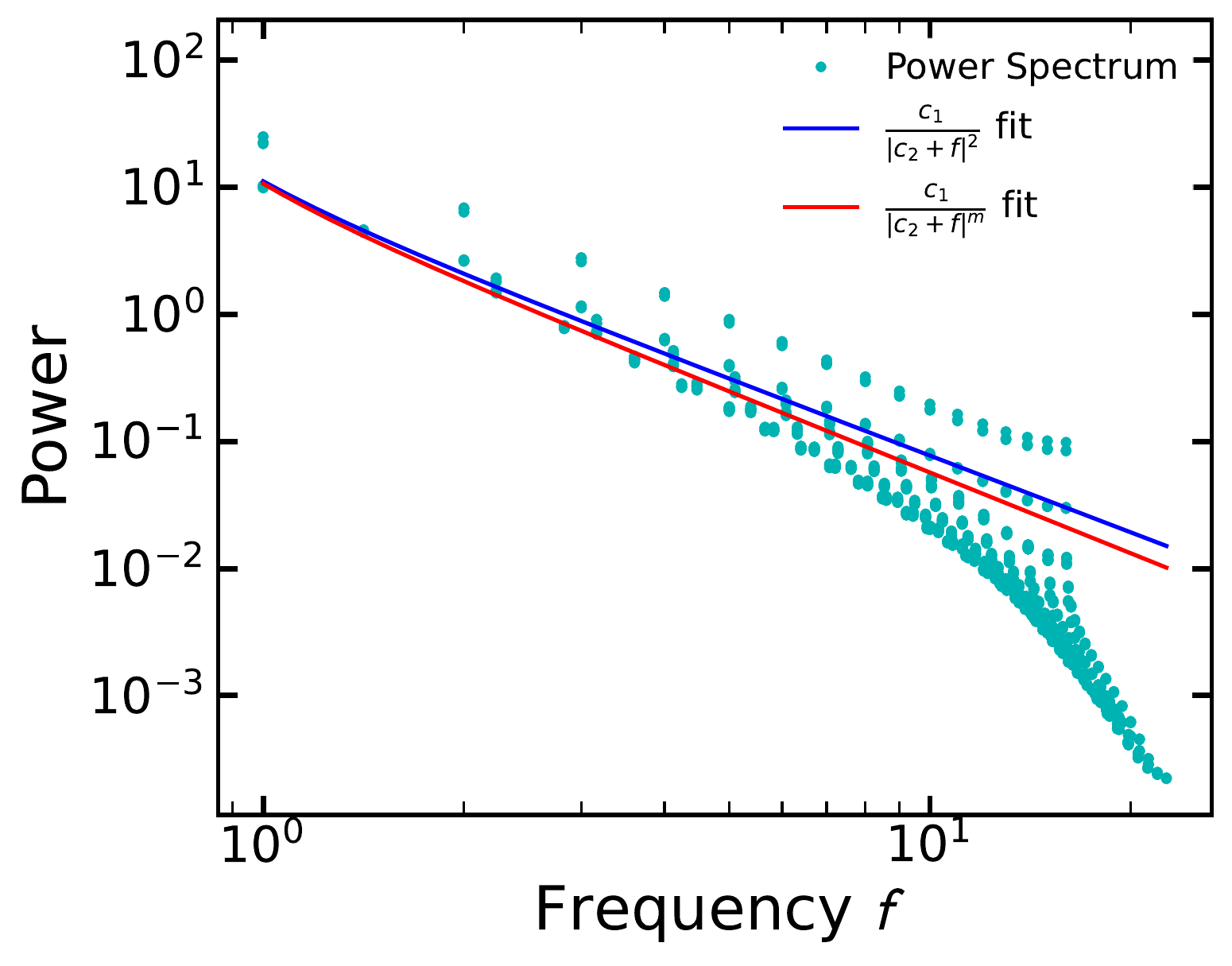}}
\def\little{\includegraphics[height=2.2cm]{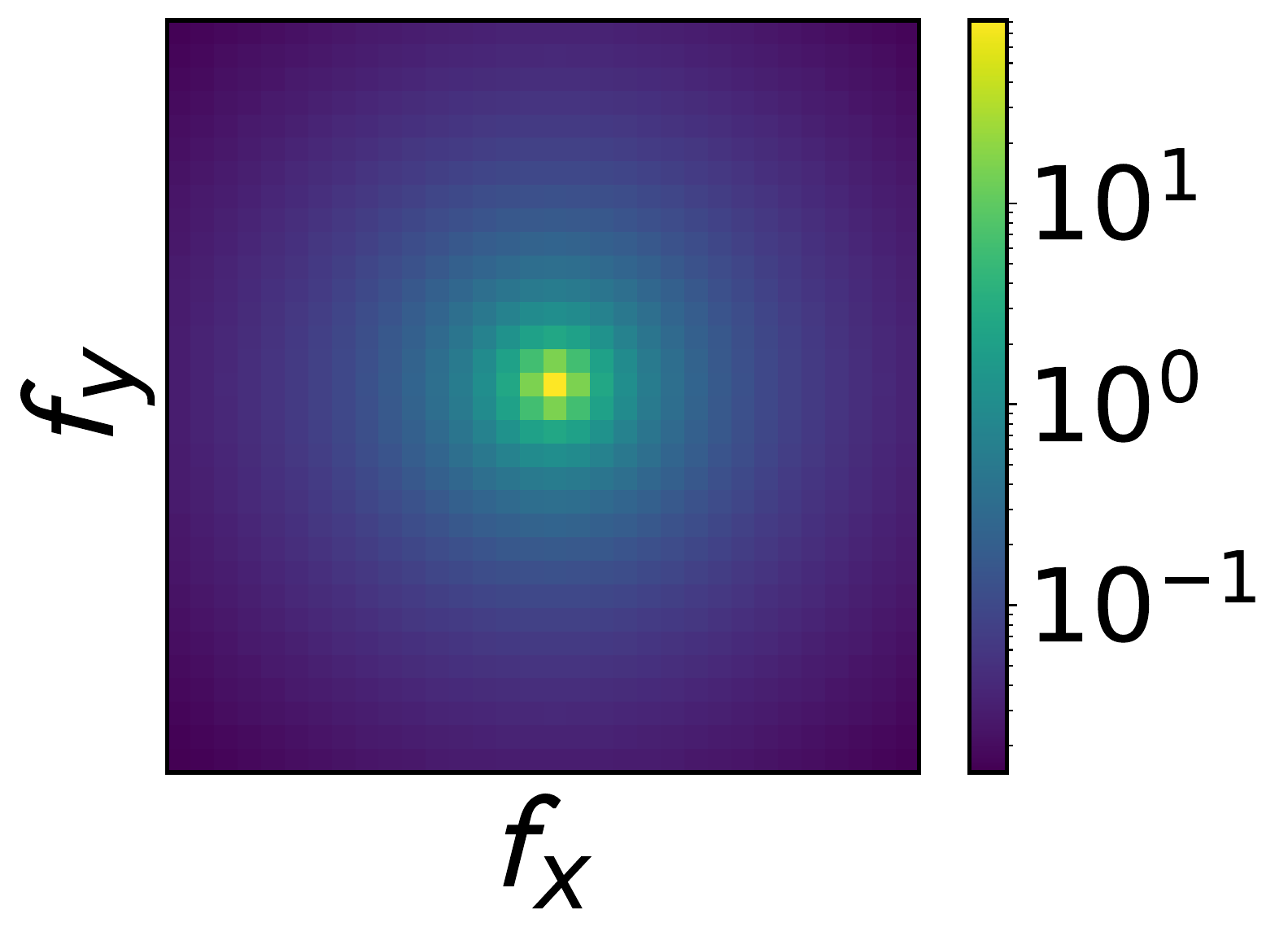}}

\begin{figure}[!th] 
\stackinset{c}{-27pt}{c}{-27pt}{\little}{\big}
\includegraphics[width=\linewidth]{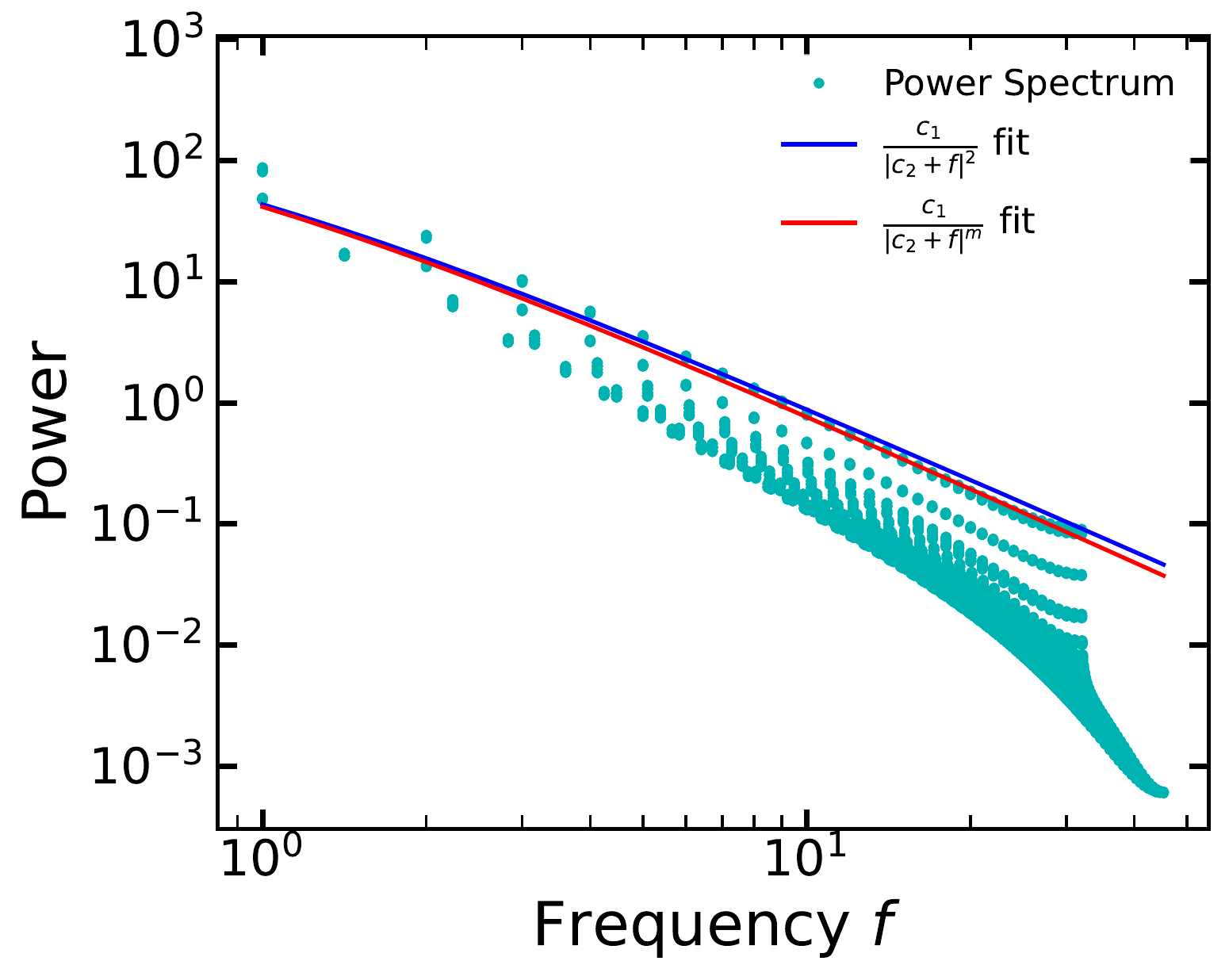}
\caption{\textbf{Power spectrum and model fit}. Dots show the empirical power spectrum, lines are fits provided by the model in equation \ref{eq:pow}.
\textbf{Top}: CIFAR10 power spectrum, with $c_1=7.7$ and $c_2=-0.3$. Inset: The 2-dimensional power spectrum obtained from the fit, as a function of the horizontal ($f_x$) and vertical ($f_y$) frequency of images.
\textbf{Bottom}: ImageNet 64x64 power spectrum, with $c_1=96.79$ and $c_2=0.49$. 
We also fit a model $\sim1/f^m$, finding $m=2.1$ for CIFAR10 and $m=2.05$ for ImageNet (while fixing $c_2=0.49$), suggesting that the inverse square model is accurate.}
    \label{fig:pow}
\end{figure}

\begin{figure}[t]
  \centering
  
  \includegraphics[width=\linewidth]{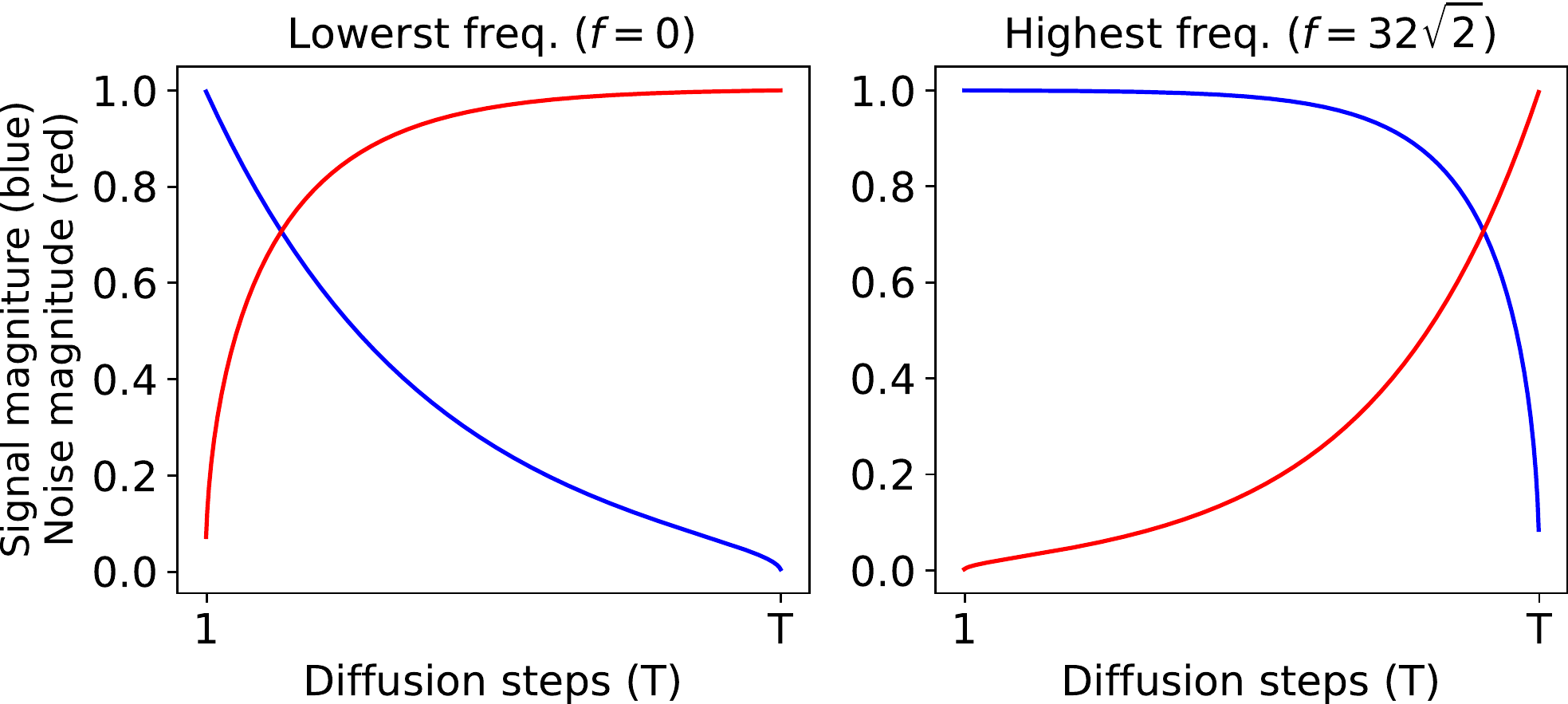}

   \caption{\textbf{Temporal dynamics of corruption for different frequencies}.
   Low frequencies change earlier (\textbf{left}), while high frequencies change later (\textbf{right}) during the corruption procedure of CIFAR10 images, for both the image (signal, blue) and the noise (red). Figure \ref{fig:path} (bottom) shows the corruption of an example image.
}
   \label{fig:speed}
\end{figure}

\section{Experiments}
\label{sec:experiments}

In this section, we validate empirically our proposed Shortest Path Diffusion (SPD) on unconditional image generation. 

We use algorithm \ref{alg:spd} for training and algorithm \ref{alg:image gen} for generating images, as described in section \ref{sec:theory}.
We conduct a range of experiments and show that the shortest path leads to the best quality of generated images in comparison to similar methods. 
Our code is available at \url{https://github.com/mtkresearch/shortest-path-diffusion}

\subsection{Dataset and metrics}

Here we describe the dataset and metrics used for our experimentation.
We use CIFAR10 \cite{cifar10dataset} and ImageNet \cite{5206848}, two of the most frequently used benchmarks for evaluating generative models on images.
CIFAR10 has resolution of $32\times 32$ pixels (dimension $d=1024$), while for ImageNet we use images scaled to $64\times 64$ resolution (dimension $d=4096$).
For both datasets, we only consider the task of \emph{unconditional} image generation.

Individual pixels are re-scaled to the range of $[-1, 1]$ following the usual practice in the literature \cite{ho2020denoising,dhariwal2021diffusion}.
We evaluate the quality of generated images by Fr\'ecet Inception Distance (FID) \cite{fid_metric}. 
We use standard practice for evaluating FID, comparing generated samples with real data and using the same Inception checkpoints as in \citet{iddpm,dhariwal2021diffusion}.
We use $50,000$ samples for CIFAR10 and $10,000$ samples for ImageNet $64 \times 64$, following \citet{iddpm}.

\subsection{Power spectrum}
\label{sec:pow}

The optimal corruption filter of SPD is obtained by the power spectrum of the dataset.
We compute the power spectrum of each image in the training set and we average the power spectrum across all images, separately for each channel.
Figure \ref{fig:pow} shows the power spectrum against frequency for CIFAR10 (top) and ImageNet $64\times64$ (bottom).
For each value of the horizontal axis we obtain multiple values of the power spectrum, corresponding to different channels and different directions of the frequency vector.
Note that frequencies are 2-dimensional vectors, the horizontal axes of figure \ref{fig:pow} show their norm.

We fit parameters $c_1$ and $c_2$ of the model in equation \ref{eq:pow} using least squares regression. 
For CIFAR10, We obtain $c_1=7.7$ and $c_2=-0.3$, while for ImageNet $64\times64$ we obtain $c_1=96.79$ and $c_2=0.49$. 
These values are used to compute the optimal corruption filter by equation \ref{eq:phifreq}.
Also shown in figure \ref{fig:pow}, we fit the exponent $m$ in equation \ref{eq:pow}, obtaining a value of $2.1$ for CIFAR10 and $2.05$ for ImageNet $64\times64$.
This confirms that the inverse square law, i.e. $m=2$, is a good model of the spectrum of natural images \cite{hyvarinen2009natural}.

\subsection{Optimal corruption filter}
\label{sec:optcor}

In this section, we investigate how images are affected by the optimal corruption filter obtained in section \ref{sec:pow}.
The linear filter $\Psi_t$ progressively dissipates the original image through equation \ref{eq:FPfreq}, but different frequencies of the original image dissipate at different times.
Similarly, different frequencies of the noise perturb the image at different times.

Figure \ref{fig:speed} shows the temporal change of signal and noise at different frequencies during corruption of CIFAR10 images.
We observe that lower frequencies dissipate first and higher frequencies dissipate later. 
Simultaneously, lower frequencies of the noise appear first, while higher frequencies appear later.
This corresponds to image sharpening and noise deblurring, and is a general property of the shortest path because equation \ref{eq:pow} is a decreasing function of $f$ and equation \ref{eq:phifreq} is a decreasing function of $D$ (for $t\in[1,N]$).
Figure \ref{fig:path} (bottom) shows the corruption of an example image.

We highlight that Shortest Path Diffusion completely determines the change of signal and noise during image corruption.
All previous studies have arbitrarily set a variety of schedules for signal and noise and tried to hyper-optimize them.
In our work, the signal and noise schedules are fixed by the optimal spatio-temporal filter $\Psi_t$.




\subsection{Training and sampling}

We use a slight modification of the codebase in \citet{dhariwal2021diffusion}.
The only difference is the corruption procedure, that we implement according to our SPD algorithm, equipped with the optimal corruption filter obtained in section \ref{sec:pow}.
For the neural network $g_{\theta}$, we use the same variant of UNet as in \cite{dhariwal2021diffusion}, without any modification. 
We optimize parameters $\theta$ by minimizing the \emph{simple} loss \cite{ho2020denoising}. 
We used Adam optimizer with learning rate of $1\times 10^{-4}$.

For CIFAR10, we use batch size $1024$, $150,000$ training iterations, and we record model checkpoints every $5,000$ iterations.
For ImageNet $64\times 64$, we use batch size $336$, $1$M training iterations, and we record model checkpoints every $3,000$ iterations.
We report the best FID score across checkpoints.
Similar to \citet{ho2020denoising}, for generating images we use Eq.~\ref{eq:phirev} with $\sigma_t = (\mbox{I}-\Psi_t\Psi_{t-1}^{-1})$ for $T>300$ and $\sigma_t = (\mbox{I}-\Psi_{t-1})(\mbox{I}-\Psi_t)^{-1}(\mbox{I}-\Psi_t\Psi_{t-1}^{-1})$ for $T\leq 300$, where $T$ is the number of diffusion steps. 

We compare SPD with other methods running on the same codebase: iDDPM and iDDPM+DDIM \cite{nichol2021improved, song2020denoising}.
Our implementation of iDDPM runs on the codebase of \citet{dhariwal2021diffusion} and gives slightly better results than the original \cite{nichol2021improved}.
We re-train SPD and iDDPM for each different value of $T$, while iDDPM+DDIM uses the deterministic DDIM sampler and is trained once at $T=4,000$ \cite{song2020denoising}.

\begin{figure}[!th]
    \centering
    \includegraphics[width=1.\linewidth]{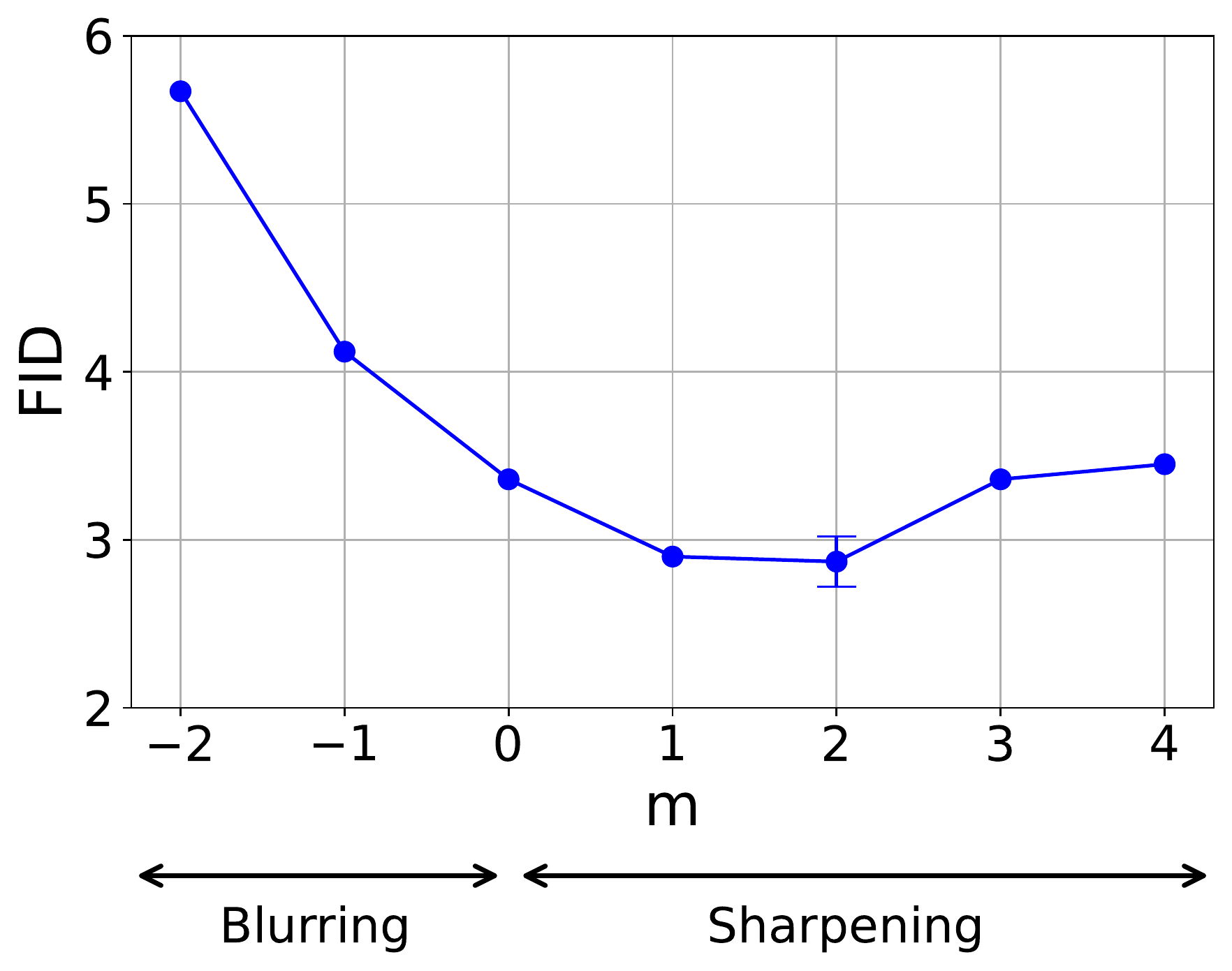}
    \caption{\textbf{Any deviation from Shortest Path Diffusion deteriorates quality of CIFAR10 images}. Image quality is measured by FID (lower is better). SPD corresponds to the value $m=2$ as shown in the power spectrum of figure \ref{fig:pow}. For comparison, we also run other corruptions, corresponding to other $m$ values. Negative and positive values of exponent $m$ result, respectively, in image blurring and image sharpening, while $m=0$ corresponds to uniform noising of all the frequencies. We found that image quality is worse in all other cases, suggesting that SPD provides the optimal corruption. We used $T=500$ diffusion timesteps in all experiments. We run $5$ experiments with different initialization for $m=2$, where standard deviation is shown, while we have single runs for all other values of $m$.}
    \label{fig:exponent}
\end{figure}

\begin{figure}[t]
    \centering
    \includegraphics[width=1.\linewidth]{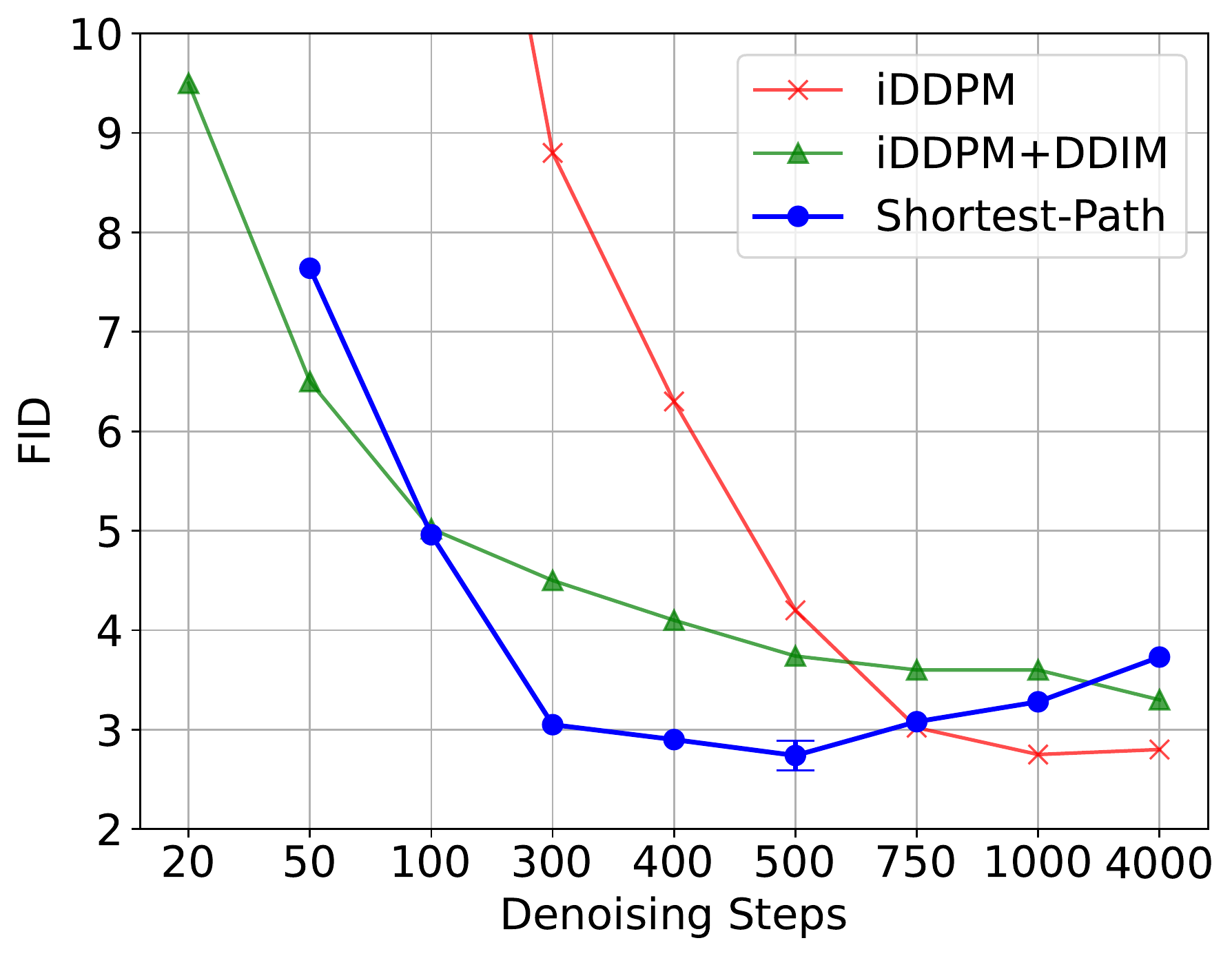}
    \caption{\textbf{CIFAR10 image quality against total diffusion timesteps}. Image quality is measured by FID (lower is better). All algorithms (iDDPM, iDDPM+DDIM, SPD) run in the same codebase. The only difference between SPD and iDDPM is the corruption procedure. SPD outperforms iDDPM+DDIM in the interval of $100 \leq T \leq 1000$ and iDDPM for $T \leq 500$.}
    \label{fig:graph}
\end{figure}

\begin{table}[b]
\centering
\caption{\textbf{Comparison of SPD with methods based on image blurring, for CIFAR10 dataset}. Image quality is measured by FID (lower is better). SPD uses $T=500$ diffusion timesteps, for all other methods we show the best FID reported by the authors in \citet{daras2022soft} (Soft Diffusion) and \citet{hoogeboom2022blurring} (Blurring Diffusion).}
\vskip 0.15in
\begin{tabular}{l|r}
\hline
Methods   & FID   \\ \hline
Soft Diffusion              & 4.64  \\
Blurring Diffusion  & 3.17   \\
\textbf{SPD (Ours)}                              & \textbf{2.74}  \\ \hline
\end{tabular}

\label{tab:overall_comparison}
\end{table}

\subsection{Results}
\label{sec:results}

First, we test our main hypothesis that Shortest Path Diffusion provides the optimal corruption.
As discussed in section \ref{sec:natural}, the optimal SPD filter depends on the power spectrum of the data, in particular the exponent $m$ in the model of equation \ref{eq:pow}, which is equal to $m=2$ for natural images.
According to our hypothesis, any other value of this exponent should result in worse performance, because it would determine a different filter and therefore a different corruption procedure. To test the optimality of shortest path, we changed $m$ in the range $[-2, 4]$ for models trained on CIFAR10. Negative and positive values of the exponent $m$ result, respectively, in image blurring and image sharpening, while $m=0$ corresponds to uniform noising of all the frequencies.
Figure \ref{fig:exponent} shows that the best performance is obtained for $m=2$, as predicted by our hypothesis.
We only test a subset of possible values for $m$, but results suggest that $m=2$ is nearly optimal.
To obtain filters at different values of $m$, we fixed $c_2$ and we set $c_1$ for each value of $m$ such that the noise variance for all filters at half-time of the forward process is equal. 


We also compare our best SPD model with other similar approaches, specifically, methods with forward noising processes containing blurring or a mixture of blurring and noising. We show in table~\ref{tab:overall_comparison} that SPD outperforms all of other methods.
Again, these other approaches provide only a small subset of all possible corruption filters, but our hypothesis that the optimal SPD filter provides the best corruption procedure still stands.
We stress that, although all methods in table~\ref{tab:overall_comparison} corrupt frequencies at different speeds, SPD sharpens images instead of blurring them.

\begin{table}[b]
\centering
\caption{\textbf{\emph{Unconditional} generation of ImageNet 64x64 images}. FID evaluations are based on 10,000 generated samples (lower is better). Results for iDDPM are copied from \citet{iddpm}. We stress that these numbers correspond to \emph{unconditional} generation, which are higher (worse) than FIDs for \emph{conditional} generation reported in other studies.}
\vskip 0.15in
\begin{tabular}{c c c c}
\hline
Methods   & Diffusion steps & Training steps & FID   \\ \hline
iDDPM  & 4000 & 1.5M & 19.2   \\
\textbf{SPD (Ours)} & 1000 & 1M & \textbf{13.7}  \\ \hline
\end{tabular}

\label{tab:imagenet_comparison}
\end{table}

In figure \ref{fig:graph}, we evaluate SPD across different values of the number of total diffusion timesteps $T$.
This number is particularly important since a smaller $T$ allows generating images faster.
We stress that, in this implementation, the only difference between SPD and iDDPM is the corruption procedure, everything else is equal (codebase, hyperparameters, computing machines).
SPD provides good image quality in a wide range of values of $T$, suggesting that it is resilient against changes of $T$ values, and outperforms iDDPM when $T \leq 500$. 
For very large $T$, we do not expect SPD to give any advantage because the errors of the reverse process may in principle vanish. 
These expectations are confirmed by observing that the advantage of SPD with respect to iDDPM occurs especially at lower values of $T$. 

We also compare SPD against iDDPM+DDIM \cite{song2020denoising}, because the latter is expected to provide better results than iDDPM at smaller $T$ values.
SPD outperforms iDDPM+DDIM for a range of values $100 \leq T \leq 1000$. 
Although SPD does not outperform iDDPM+DDIM below $T<100$, the FID for both methods are poor enough to be discarded. 
Figure \ref{fig:qual_results} in appendix \ref{app:gen images} shows example images generated by SPD, iDDPM and iDDPM+DDIM.

Lastly, in table \ref{tab:imagenet_comparison} we show quality of \emph{unconditional} generation of ImageNet $64\times64$  images, in comparison with iDDPM \cite{iddpm}. 
Our model outperforms iDDPM despite having a lower number of diffusion steps and training for fewer iterations.
We stress that FIDs in table \ref{tab:imagenet_comparison} correspond to \emph{unconditional} generation, which are higher (worse) than FIDs for \emph{conditional} generation reported in other studies.
Figure \ref{fig:IN64images} in appendix \ref{app:gen images} shows example images generated from the model.

\section{Conclusions}
\label{sec:conclusions}

We introduced Shortest Path Diffusion (SPD), a Diffusion Model providing a unique procedure for data corruption.
Previous work explored different procedures for corrupting data and tried to optimize them empirically.
Instead, we argue that SPD provides the optimal corruption since it minimizes the length of the path taken by the corruption in the space of probability distributions.
Although we do not provide any proof of optimality, we argue that taking the shortest path may reduce the effect of errors in estimating the reverse transition probabilities. 
Interestingly, while previous work explored image blurring, instead we found that image sharpening provides better results.

In contrast to previous work, the corruption of SPD is data-dependent, thus SPD provides the flexibility of adjusting the corruption to the given dataset (but see \citet{lee2022priorgrad} for a similar approach).
Furthermore, SPD can be applied not only to images but also to other types of data.
However, different types of data may require a slightly different treatment in order to make SPD feasible.
In general, SPD requires computing the full covariance matrix of the data, which scales quadratically with its dimension.
However, this complexity can be reduced to linear if a strong prior on the form of the covariance is given.

For natural images, whose distribution is approximately translation invariant, SPD requires computing the power spectrum of the data, which scales linearly with either the dimension or the size of dataset.
The same approach may apply to several other image datasets (e.g. LSUN).
Furthermore, a similar approach may apply to non-image data that is nevertheless translation-invariant, such as audio and speech data.
However, other types of data may require a different treatment, including non-translation invariant image datasets (e.g. CelebA).

A limitation of our work is that the shortest path is computed in closed form for Gaussian distributions only, while most distributions of interest, including real images, are not Gaussian.
We hypothesized that the shortest path for real images could be approximated by that of a Gaussian distribution with the same covariance.
We also chose the Fisher metric to compute the shortest path, but other choices are possible (e.g. Wasserstein metric, see \citet{khrulkov2022understanding}).
We found strong empirical support for the assumptions of this study, but more research is required to test them further, for example on higher-resolution images, other types of data or metrics.

SPD mostly concerns the corruption procedure, it is orthogonal to studies that improved Diffusion Models by other means, and those studies may be used to further improve SPD.
Those include, for example, using dedicated ODE solvers for sampling from the model \cite{lu2022dpmsolver}, more accurate estimation of the covariance of the reverse transition probability \cite{bao2022analytic, bao2022estimating}, using auxiliary variables \cite{dockhorn2021score}, learning the optimal reverse steps \cite{salimans2022progressive}, and diffusing in latent space \cite{vahdat2021score,jing2022subspace}.
Some of this studies obtain FID scores lower (better) than our work, but we believe that they could be improved further by implementing our proposed SPD corruption. 
We believe that SPD provides a useful tool to advance the progress of Diffusion Models in generating a variety of different types of data.

\bibliography{paper_bib}

\begin{thebibliography}{49}
\providecommand{\natexlab}[1]{#1}
\providecommand{\url}[1]{\texttt{#1}}
\expandafter\ifx\csname urlstyle\endcsname\relax
  \providecommand{\doi}[1]{doi: #1}\else
  \providecommand{\doi}{doi: \begingroup \urlstyle{rm}\Url}\fi

\bibitem[Amari(2016)]{amari2016information}
Amari, S.-i.
\newblock \emph{Information geometry and its applications}, volume 194.
\newblock Springer, 2016.

\bibitem[Bansal et~al.(2022)Bansal, Borgnia, Chu, Li, Kazemi, Huang, Goldblum,
  Geiping, and Goldstein]{bansal2022cold}
Bansal, A., Borgnia, E., Chu, H.-M., Li, J.~S., Kazemi, H., Huang, F.,
  Goldblum, M., Geiping, J., and Goldstein, T.
\newblock Cold diffusion: Inverting arbitrary image transforms without noise,
  2022.

\bibitem[Bao et~al.(2022{\natexlab{a}})Bao, Li, Sun, Zhu, and
  Zhang]{bao2022estimating}
Bao, F., Li, C., Sun, J., Zhu, J., and Zhang, B.
\newblock Estimating the optimal covariance with imperfect mean in diffusion
  probabilistic models.
\newblock \emph{arXiv preprint arXiv:2206.07309}, 2022{\natexlab{a}}.

\bibitem[Bao et~al.(2022{\natexlab{b}})Bao, Li, Zhu, and
  Zhang]{bao2022analytic}
Bao, F., Li, C., Zhu, J., and Zhang, B.
\newblock Analytic-dpm: an analytic estimate of the optimal reverse variance in
  diffusion probabilistic models.
\newblock \emph{arXiv preprint arXiv:2201.06503}, 2022{\natexlab{b}}.

\bibitem[Bengio et~al.(2014)Bengio, Laufer, Alain, and
  Yosinski]{bengio2014deep}
Bengio, Y., Laufer, E., Alain, G., and Yosinski, J.
\newblock Deep generative stochastic networks trainable by backprop.
\newblock In \emph{International Conference on Machine Learning}, pp.\
  226--234. PMLR, 2014.

\bibitem[Benny \& Wolf(2022)Benny and Wolf]{benny2022dynamic}
Benny, Y. and Wolf, L.
\newblock Dynamic dual-output diffusion models.
\newblock In \emph{Proceedings of the IEEE/CVF Conference on Computer Vision
  and Pattern Recognition}, pp.\  11482--11491, 2022.

\bibitem[Bond-Taylor et~al.(2022)Bond-Taylor, Leach, Long, and
  Willcocks]{9555209}
Bond-Taylor, S., Leach, A., Long, Y., and Willcocks, C.~G.
\newblock Deep generative modelling: A comparative review of vaes, gans,
  normalizing flows, energy-based and autoregressive models.
\newblock \emph{IEEE Transactions on Pattern Analysis and Machine
  Intelligence}, 44\penalty0 (11):\penalty0 7327--7347, 2022.
\newblock \doi{10.1109/TPAMI.2021.3116668}.

\bibitem[Chen et~al.(2019)Chen, Behrmann, Duvenaud, and
  Jacobsen]{chen2019residual}
Chen, R.~T., Behrmann, J., Duvenaud, D.~K., and Jacobsen, J.-H.
\newblock Residual flows for invertible generative modeling.
\newblock \emph{Advances in Neural Information Processing Systems}, 32, 2019.

\bibitem[Child(2021)]{child2020very}
Child, R.
\newblock Very deep vaes generalize autoregressive models and can outperform
  them on images.
\newblock \emph{ICLR}, 2021.

\bibitem[Child et~al.(2019)Child, Gray, Radford, and
  Sutskever]{child2019generating}
Child, R., Gray, S., Radford, A., and Sutskever, I.
\newblock Generating long sequences with sparse transformers.
\newblock \emph{arXiv preprint arXiv:1904.10509}, 2019.

\bibitem[Daras et~al.(2022)Daras, Delbracio, Talebi, Dimakis, and
  Milanfar]{daras2022soft}
Daras, G., Delbracio, M., Talebi, H., Dimakis, A.~G., and Milanfar, P.
\newblock Soft diffusion: Score matching for general corruptions.
\newblock \emph{arXiv preprint arXiv:2209.05442}, 2022.

\bibitem[Deng et~al.(2009)Deng, Dong, Socher, Li, Li, and Fei-Fei]{5206848}
Deng, J., Dong, W., Socher, R., Li, L.-J., Li, K., and Fei-Fei, L.
\newblock Imagenet: A large-scale hierarchical image database.
\newblock In \emph{2009 IEEE Conference on Computer Vision and Pattern
  Recognition}, pp.\  248--255, 2009.
\newblock \doi{10.1109/CVPR.2009.5206848}.

\bibitem[Dhariwal \& Nichol(2021)Dhariwal and Nichol]{dhariwal2021diffusion}
Dhariwal, P. and Nichol, A.
\newblock Diffusion models beat gans on image synthesis.
\newblock \emph{Advances in Neural Information Processing Systems},
  34:\penalty0 8780--8794, 2021.

\bibitem[Dinh et~al.(2017)Dinh, Sohl-Dickstein, and Bengio]{dinh2016density}
Dinh, L., Sohl-Dickstein, J., and Bengio, S.
\newblock Density estimation using real nvp.
\newblock \emph{ICLR}, 2017.

\bibitem[Dockhorn et~al.(2021)Dockhorn, Vahdat, and Kreis]{dockhorn2021score}
Dockhorn, T., Vahdat, A., and Kreis, K.
\newblock Score-based generative modeling with critically-damped langevin
  diffusion.
\newblock \emph{arXiv preprint arXiv:2112.07068}, 2021.

\bibitem[Fox(1987)]{fox1987introduction}
Fox, C.
\newblock \emph{An introduction to the calculus of variations}.
\newblock Courier Corporation, 1987.

\bibitem[Goodfellow et~al.(2020)Goodfellow, Pouget-Abadie, Mirza, Xu,
  Warde-Farley, Ozair, Courville, and Bengio]{goodfellow2020generative}
Goodfellow, I., Pouget-Abadie, J., Mirza, M., Xu, B., Warde-Farley, D., Ozair,
  S., Courville, A., and Bengio, Y.
\newblock Generative adversarial networks.
\newblock \emph{Communications of the ACM}, 63\penalty0 (11):\penalty0
  139--144, 2020.

\bibitem[Guth et~al.(2022)Guth, Coste, Bortoli, and Mallat]{guth2022wavelet}
Guth, F., Coste, S., Bortoli, V.~D., and Mallat, S.
\newblock Wavelet score-based generative modeling.
\newblock In Oh, A.~H., Agarwal, A., Belgrave, D., and Cho, K. (eds.),
  \emph{Advances in Neural Information Processing Systems}, 2022.
\newblock URL \url{https://openreview.net/forum?id=xZmjH3Pm2BK}.

\bibitem[Heusel et~al.(2017)Heusel, Ramsauer, Unterthiner, Nessler, and
  Hochreiter]{fid_metric}
Heusel, M., Ramsauer, H., Unterthiner, T., Nessler, B., and Hochreiter, S.
\newblock Gans trained by a two time-scale update rule converge to a local nash
  equilibrium.
\newblock In \emph{Advances in Neural Information Processing Systems 30: Annual
  Conference on Neural Information Processing Systems 2017, December 4-9, 2017,
  Long Beach, CA, {USA}}, 2017.

\bibitem[Ho et~al.(2020)Ho, Jain, and Abbeel]{ho2020denoising}
Ho, J., Jain, A., and Abbeel, P.
\newblock Denoising diffusion probabilistic models.
\newblock \emph{Advances in Neural Information Processing Systems},
  33:\penalty0 6840--6851, 2020.

\bibitem[Hoogeboom \& Salimans(2023)Hoogeboom and
  Salimans]{hoogeboom2022blurring}
Hoogeboom, E. and Salimans, T.
\newblock Blurring diffusion models.
\newblock \emph{ICLR}, 2023.

\bibitem[Horn \& Johnson(2013)Horn and Johnson]{horn13}
Horn, R.~A. and Johnson, C.~R.
\newblock \emph{Matrix Analysis}.
\newblock Cambridge University Press, Cambridge; New York, 2nd edition, 2013.
\newblock ISBN 9780521839402.

\bibitem[Hyv{\"a}rinen et~al.(2009)Hyv{\"a}rinen, Hurri, and
  Hoyer]{hyvarinen2009natural}
Hyv{\"a}rinen, A., Hurri, J., and Hoyer, P.~O.
\newblock \emph{Natural image statistics: A probabilistic approach to early
  computational vision.}, volume~39.
\newblock Springer Science \& Business Media, 2009.

\bibitem[Jing et~al.(2022)Jing, Corso, Berlinghieri, and
  Jaakkola]{jing2022subspace}
Jing, B., Corso, G., Berlinghieri, R., and Jaakkola, T.
\newblock Subspace diffusion generative models.
\newblock \emph{arXiv preprint arXiv:2205.01490}, 2022.

\bibitem[Karras et~al.(2020)Karras, Laine, Aittala, Hellsten, Lehtinen, and
  Aila]{karras2020analyzing}
Karras, T., Laine, S., Aittala, M., Hellsten, J., Lehtinen, J., and Aila, T.
\newblock Analyzing and improving the image quality of stylegan.
\newblock In \emph{Proceedings of the IEEE/CVF conference on computer vision
  and pattern recognition}, pp.\  8110--8119, 2020.

\bibitem[Khrulkov et~al.(2023)Khrulkov, Ryzhakov, Chertkov, and
  Oseledets]{khrulkov2022understanding}
Khrulkov, V., Ryzhakov, G., Chertkov, A., and Oseledets, I.
\newblock Understanding ddpm latent codes through optimal transport.
\newblock \emph{ICLR}, 2023.

\bibitem[Kingma \& Welling(2014)Kingma and Welling]{kingma2014auto}
Kingma, D. and Welling, M.
\newblock Auto-encoding variational bayes.
\newblock \emph{ICLR}, 2014.

\bibitem[Kingma et~al.(2021)Kingma, Salimans, Poole, and
  Ho]{kingma2021variational}
Kingma, D., Salimans, T., Poole, B., and Ho, J.
\newblock Variational diffusion models.
\newblock \emph{Advances in neural information processing systems},
  34:\penalty0 21696--21707, 2021.

\bibitem[Krizhevsky(2009)]{cifar10dataset}
Krizhevsky, A.
\newblock Learning multiple layers of features from tiny images.
\newblock \emph{Master's thesis, University of Tront}, 2009.

\bibitem[Lee et~al.(2022{\natexlab{a}})Lee, Chung, Kim, and
  Ye]{lee2022progressive}
Lee, S., Chung, H., Kim, J., and Ye, J.~C.
\newblock Progressive deblurring of diffusion models for coarse-to-fine image
  synthesis.
\newblock \emph{arXiv preprint arXiv:2207.11192}, 2022{\natexlab{a}}.

\bibitem[Lee et~al.(2022{\natexlab{b}})Lee, Kim, Shin, Tan, Liu, Meng, Qin,
  Chen, Yoon, and Liu]{lee2022priorgrad}
Lee, S.-G., Kim, H., Shin, C., Tan, X., Liu, C., Meng, Q., Qin, T., Chen, W.,
  Yoon, S., and Liu, T.-Y.
\newblock Priorgrad: Improving conditional denoising diffusion models with
  data-dependent adaptive prior.
\newblock In \emph{International Conference on Learning Representations},
  2022{\natexlab{b}}.
\newblock URL \url{https://openreview.net/forum?id=_BNiN4IjC5}.

\bibitem[Lu et~al.(2022)Lu, Zhou, Bao, Chen, Li, and Zhu]{lu2022dpmsolver}
Lu, C., Zhou, Y., Bao, F., Chen, J., Li, C., and Zhu, J.
\newblock {DPM}-solver: A fast {ODE} solver for diffusion probabilistic model
  sampling in around 10 steps.
\newblock In Oh, A.~H., Agarwal, A., Belgrave, D., and Cho, K. (eds.),
  \emph{Advances in Neural Information Processing Systems}, 2022.
\newblock URL \url{https://openreview.net/forum?id=2uAaGwlP_V}.

\bibitem[Ma et~al.(2022)Ma, Zhang, Zhu, and Feng]{ma2022accelerating}
Ma, H., Zhang, L., Zhu, X., and Feng, J.
\newblock Accelerating score-based generative models with preconditioned
  diffusion sampling.
\newblock In \emph{European Conference on Computer Vision}, pp.\  1--16.
  Springer, 2022.

\bibitem[Nichol et~al.(2022)Nichol, Dhariwal, Ramesh, Shyam, Mishkin, McGrew,
  Sutskever, and Chen]{nichol2021glide}
Nichol, A., Dhariwal, P., Ramesh, A., Shyam, P., Mishkin, P., McGrew, B.,
  Sutskever, I., and Chen, M.
\newblock Glide: Towards photorealistic image generation and editing with
  text-guided diffusion models.
\newblock \emph{International Conference on Machine Learning}, 2022.

\bibitem[Nichol \& Dhariwal(2021{\natexlab{a}})Nichol and Dhariwal]{iddpm}
Nichol, A.~Q. and Dhariwal, P.
\newblock Improved denoising diffusion probabilistic models.
\newblock In \emph{Proceedings of the 38th International Conference on Machine
  Learning, {ICML} 2021, 18-24 July 2021, Virtual Event}, volume 139 of
  \emph{Proceedings of Machine Learning Research}, pp.\  8162--8171,
  2021{\natexlab{a}}.

\bibitem[Nichol \& Dhariwal(2021{\natexlab{b}})Nichol and
  Dhariwal]{nichol2021improved}
Nichol, A.~Q. and Dhariwal, P.
\newblock Improved denoising diffusion probabilistic models.
\newblock In \emph{International Conference on Machine Learning}, pp.\
  8162--8171. PMLR, 2021{\natexlab{b}}.

\bibitem[Pinele et~al.(2020)Pinele, Strapasson, and Costa]{pinele2020fisher}
Pinele, J., Strapasson, J.~E., and Costa, S.~I.
\newblock The fisher--rao distance between multivariate normal distributions:
  Special cases, bounds and applications.
\newblock \emph{Entropy}, 22\penalty0 (4):\penalty0 404, 2020.

\bibitem[Rissanen et~al.(2023)Rissanen, Heinonen, and
  Solin]{rissanen2022generative}
Rissanen, S., Heinonen, M., and Solin, A.
\newblock Generative modelling with inverse heat dissipation.
\newblock \emph{ICLR}, 2023.

\bibitem[Saharia et~al.(2022)Saharia, Chan, Saxena, Li, Whang, Denton,
  Ghasemipour, Ayan, Mahdavi, Lopes, et~al.]{saharia2022photorealistic}
Saharia, C., Chan, W., Saxena, S., Li, L., Whang, J., Denton, E., Ghasemipour,
  S. K.~S., Ayan, B.~K., Mahdavi, S.~S., Lopes, R.~G., et~al.
\newblock Photorealistic text-to-image diffusion models with deep language
  understanding.
\newblock \emph{arXiv preprint arXiv:2205.11487}, 2022.

\bibitem[Salimans \& Ho(2022)Salimans and Ho]{salimans2022progressive}
Salimans, T. and Ho, J.
\newblock Progressive distillation for fast sampling of diffusion models.
\newblock \emph{arXiv preprint arXiv:2202.00512}, 2022.

\bibitem[Salimans et~al.(2015)Salimans, Kingma, and
  Welling]{salimans2015markov}
Salimans, T., Kingma, D., and Welling, M.
\newblock Markov chain monte carlo and variational inference: Bridging the gap.
\newblock In \emph{International conference on machine learning}, pp.\
  1218--1226. PMLR, 2015.

\bibitem[Sohl-Dickstein et~al.(2015)Sohl-Dickstein, Weiss, Maheswaranathan, and
  Ganguli]{sohl2015deep}
Sohl-Dickstein, J., Weiss, E., Maheswaranathan, N., and Ganguli, S.
\newblock Deep unsupervised learning using nonequilibrium thermodynamics.
\newblock In \emph{International Conference on Machine Learning}, pp.\
  2256--2265. PMLR, 2015.

\bibitem[Song et~al.(2021)Song, Meng, and Ermon]{song2020denoising}
Song, J., Meng, C., and Ermon, S.
\newblock Denoising diffusion implicit models.
\newblock \emph{ICLR}, 2021.

\bibitem[Song \& Ermon(2019)Song and Ermon]{song2019generative}
Song, Y. and Ermon, S.
\newblock Generative modeling by estimating gradients of the data distribution.
\newblock \emph{Advances in Neural Information Processing Systems}, 32, 2019.

\bibitem[Song \& Ermon(2020)Song and Ermon]{song2020improved}
Song, Y. and Ermon, S.
\newblock Improved techniques for training score-based generative models.
\newblock \emph{Advances in neural information processing systems},
  33:\penalty0 12438--12448, 2020.

\bibitem[Song et~al.(2020)Song, Sohl-Dickstein, Kingma, Kumar, Ermon, and
  Poole]{song2020score}
Song, Y., Sohl-Dickstein, J., Kingma, D.~P., Kumar, A., Ermon, S., and Poole,
  B.
\newblock Score-based generative modeling through stochastic differential
  equations.
\newblock \emph{arXiv preprint arXiv:2011.13456}, 2020.

\bibitem[Vahdat et~al.(2021)Vahdat, Kreis, and Kautz]{vahdat2021score}
Vahdat, A., Kreis, K., and Kautz, J.
\newblock Score-based generative modeling in latent space.
\newblock \emph{Advances in Neural Information Processing Systems},
  34:\penalty0 11287--11302, 2021.

\bibitem[Van Den~Oord et~al.(2016)Van Den~Oord, Kalchbrenner, and
  Kavukcuoglu]{van2016pixel}
Van Den~Oord, A., Kalchbrenner, N., and Kavukcuoglu, K.
\newblock Pixel recurrent neural networks.
\newblock In \emph{International conference on machine learning}, pp.\
  1747--1756. PMLR, 2016.

\bibitem[Watson et~al.(2021)Watson, Chan, Ho, and Norouzi]{watson2021learning}
Watson, D., Chan, W., Ho, J., and Norouzi, M.
\newblock Learning fast samplers for diffusion models by differentiating
  through sample quality.
\newblock In \emph{International Conference on Learning Representations}, 2021.

\end{thebibliography}
\bibliographystyle{icml2023}

\newpage
\appendix
\onecolumn
\section{Shortest path for Gaussian distributions}
\label{app:fisher proof}

\begin{theorem}
\label{th:sp}
Given two Gaussian distributions with zero mean and covariance matrix equal to, respectively, $\Sigma_0$ and $\Sigma_1$, where $\Sigma_1$ is non-singular.
Given the Riemannian metric defined by the Fisher information, the shortest path between the two distributions is given by
\begin{equation}
\Sigma_t=\Sigma_1^{1/2}\left(\Sigma_1^{-1/2}\Sigma_0\Sigma_1^{-1/2}\right)^{1-t}\Sigma_1^{1/2}
\end{equation}
where $t\in(0,1)$ measures the relative distance travelled along the path.

\end{theorem}

\begin{remark} \label{rem:identity}
In the special case of $\Sigma_1=\mbox{\emph{I}}$, theorem \ref{th:sp} implies that
\begin{equation}
\Sigma_t=\Sigma_0^{1-t}
\end{equation}
Therefore the eigenvectors of $\Sigma_t$ are equal to those of $\Sigma_0$, and the eigenvalues of $\Sigma_t$ are equal to those of $\Sigma_0$ raised to the power of $1-t$.  
Denoting by $\sigma_x^2$ and $\sigma_y^2$ any pair of eigenvalues of $\Sigma_t$, along the shortest path they satisfy the following equation 
\begin{equation}\label{eq:eigs}
    \sigma_y^2=\left(\sigma_x^2\right)^{\alpha}
\end{equation}
where $\alpha$ depends on the eigenvalues at $t=0$.
Shortest paths are illustrated in figure \ref{fig:short-pathes}.

\begin{figure}[t]
    \centering
    \includegraphics[width=0.6\linewidth]{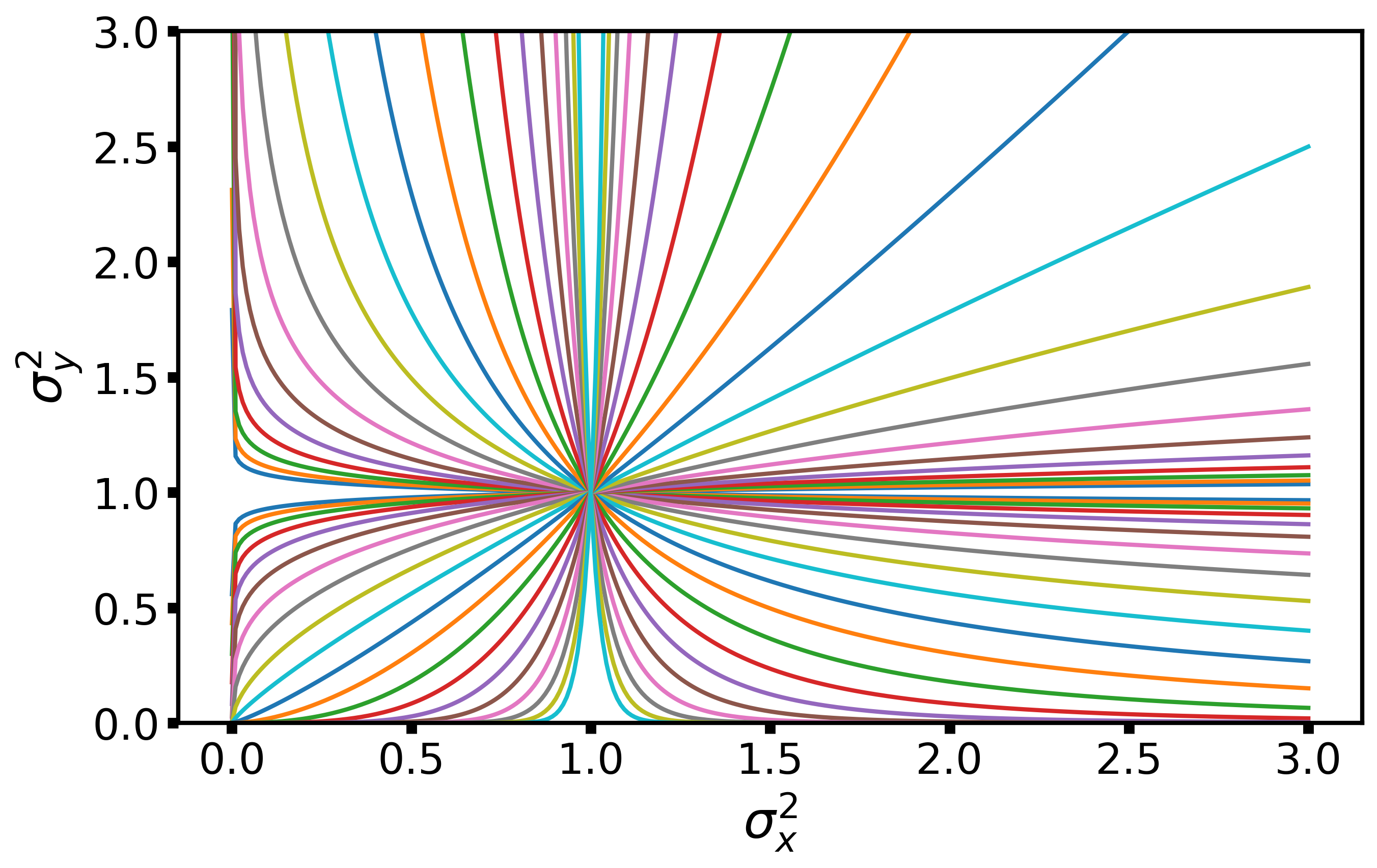}
    \caption{\textbf{Shortest path between Gaussians using Fisher metric.} Eigenvalues $\sigma_x^2$ and $\sigma_y^2$ of the covariance matrix follow equation \ref{eq:eigs}.}
    \label{fig:short-pathes}
\end{figure}

\end{remark}

\begin{proof}

Theorem \ref{th:sp} is discussed in \citet{pinele2020fisher} and references therein, here we provide an alternative derivation for completeness.
The length of a curve is described by the sum of the lengths of its infinitesimal segments, where the length of a given segment $d\mathbf{x}$ is equal to its Eucledian norm $\left|d\mathbf{x}\right|$.
The length of a curve $\mathbf{x}(s)\in\mathbb{R}^d$, parameterized by a scalar $s$, is equal to the line integral
\begin{equation}
    \mathcal{D} = \int_{s_0}^{s_1} ds \left|\frac{d\mathbf{x}}{ds}\right|= \int_{s_0}^{s_1} ds\;\left(\sum_{i=1}^{d}\frac{dx_i}{ds}^2\right)^{1/2}
\end{equation}
where the initial and final points of the curve are represented by $\mathbf{x}(s_0)$ and $\mathbf{x}(s_1)$, respectively. 
When measuring distances using a non-Euclidean (Riemannian) metric tensor $\mathcal{I}\in\mathbb R^{d\times d}$, the length of the curve is equal to
\begin{equation}
    \mathcal{D} = \int_{s_0}^{s_1} ds\; \left(\sum_{i,j=1}^{d}\frac{dx_i}{ds}\mathcal{I}_{ij}\frac{dx_j}{ds}\right)^{1/2}
\end{equation}
In the case of Gaussian distributions with zero mean, the curve is represented by the covariance $\Sigma(s)\in\mathbb R^{d\times d}$.
The Riemannian metric is given by the tensor $\mathcal{I}\in\mathbb R^{(d\times d)\times (d\times d)}$ and the length of the curve is given by
\begin{equation}
    \mathcal{D} = \int_{s_0}^{s_1} ds\; \left(\sum_{i,j,k,l}^{1,d}\frac{d\Sigma_{ij}}{ds}\mathcal{I}_{ij,kl}\frac{d\Sigma_{kl}}{ds}\right)^{1/2}
\end{equation}
We use the Fisher information matrix as the metric tensor.
For a Gaussians distribution with zero mean and covariance $\Sigma$, the Fisher information matrix is equal to
\begin{equation}
F_{ij,kl}=\frac{1}{2}\left(\Sigma^{-1}\right)_{il}\left(\Sigma^{-1}\right)_{jk}
\end{equation}
Thus, the length of the curve is equal to
\begin{equation}
\label{eq:D}
    \mathcal{D} = \frac{1}{\sqrt{2}}\int_{s_0}^{s_1} ds\; \text{Tr}\left(\Sigma^{-1}\frac{d\Sigma}{ds}\Sigma^{-1}\frac{d\Sigma}{ds}\right)^{1/2}
\end{equation}
It is important to note that, given the properties of the Fisher metric, the length of the curve is \emph{reparameterization invariant}. 
This means that it depends only on the distribution and not on how the distribution is parameterized \cite{amari2016information}.

Our aim is to find the curve of shortest length, namely the curve $\Sigma(s)$ that minimizes equation \ref{eq:D}. 
All stationary points of \ref{eq:D} satisfy the Euler-Lagrange equation, which is given by \cite{fox1987introduction}
\begin{equation} \label{eq:eu-lag}
    \frac{\partial \mathcal{L}}{\partial \Sigma} = \frac{d}{ds}\frac{\partial\mathcal{L}}{\partial\dot{\Sigma}}
\end{equation}
where $\dot{\Sigma} = \frac{d\Sigma}{ds}$ and the Lagrangian is equal to 
\begin{equation}
\mathcal{L}(\Sigma(s), \dot{\Sigma}(s)) = \frac{1}{\sqrt{2}} \text{Tr}(\Sigma^{-1}\dot{\Sigma}\;\Sigma^{-1}\dot{\Sigma})^{1/2}\end{equation}
After taking gradients of the Lagrangian with respect to $\Sigma$ and $\dot{\Sigma}$, equation \ref{eq:eu-lag} becomes
\begin{equation} \label{eq:lag-cov}
    -\frac{1}{2\mathcal{L}}\Sigma^{-1}\dot{\Sigma}\Sigma^{-1}\dot{\Sigma}\Sigma^{-1} = \frac{d}{ds}\left(\frac{1}{2\mathcal{L}}\Sigma^{-1}\dot{\Sigma}\Sigma^{-1}\right)
\end{equation}
This equation can be simplified by a sequence of steps.
First, we multiply both sides by the scalar $\frac{2}{\mathcal{L}}$, and we obtain
\begin{equation}
    -\Sigma^{-1}\frac{d\Sigma}{\mathcal{L}ds}\Sigma^{-1}\frac{d\Sigma}{\mathcal{L}ds}\Sigma^{-1}=\frac{d}{\mathcal{L}ds}\left(\Sigma^{-1}\frac{d\Sigma}{\mathcal{L}ds}\Sigma^{-1}\right)
\end{equation}
Second, we make the change of variable $dt = \mathcal{L}ds$ and we obtain
\begin{equation}
\label{eq:ao}
    -\Sigma^{-1}\frac{d\Sigma}{dt}\Sigma^{-1}\frac{d\Sigma}{dt}\Sigma^{-1}=\frac{d}{dt}\left(\Sigma^{-1}\frac{d\Sigma}{dt}\Sigma^{-1}\right)
\end{equation}
Here we slightly abuse notation since $\Sigma$ is now a function of $t$ instead of $s$.
Note that the Lagrangian $\mathcal{L}$ depends on $s$ through $\Sigma$ and $\dot{\Sigma}$, therefore, $t$ is a nonlinear function of $s$ in general.
Third, we use the identity $d\left(\Sigma^{-1}\right)=-\Sigma^{-1}(d\Sigma)\Sigma^{-1}$, and we multiply both sides of equation \ref{eq:ao} by $\Sigma$.
We arrive at
\begin{equation}
\label{eq:spdiffeq}
    \frac{d^2\Sigma}{dt^2} = \frac{d\Sigma}{dt}\Sigma^{-1}\frac{d\Sigma}{dt}
\end{equation}
The shortest path can be found by solving this second-order differential equation with boundary conditions $\Sigma(t_0)=\Sigma_0$ and $\Sigma(t_1)=\Sigma_1$.

The problem can be further simplified by noting that equation \ref{eq:spdiffeq} is invariant for congruent transformations
\begin{equation}
\label{eq:cong}
\Sigma(t)=F\Sigma'(t) F^\dagger
\end{equation}
where $F$ is any non-singular matrix, which does \emph{not} have to be orthogonal.
Remarkably, we can choose the matrix $F$ in a way that $\Sigma'(t)$ is diagonal along the entire path, from beginning to end.
As proved by Theorem 7.6.4 in \cite{horn13}, given two positive definite matrices $\Sigma_0$ and $\Sigma_1$, there is a non-singular matrix $F$ and diagonal matrix $D$, such that
\begin{align}\label{eq:boundry}
    \Sigma_0&=FDF^\dagger\\
    \Sigma_1&=FF^\dagger
\end{align}
The matrix $F$ is given by $F=\Sigma_1^{1/2}U$, where the columns of $U$ are the orthogonal eigenvectors of the matrix $\Sigma_1^{-1/2}\Sigma_0\Sigma_1^{-1/2}$, and $D$ is the diagonal matrix of its eigenvalues.
It is straightforward to verify that $FF^\dagger=\Sigma_1^{1/2}UU^\dagger\Sigma_1^{1/2}=\Sigma_1$.
Furthermore, we have that $FDF^\dagger=\Sigma_1^{1/2}UDU^\dagger\Sigma_1^{1/2}=\Sigma_1^{1/2}\Sigma_1^{-1/2}\Sigma_0\Sigma_1^{-1/2}\Sigma_1^{1/2}=\Sigma_0$.

Under the congruent transformation \ref{eq:cong}, the shortest path equation \ref{eq:spdiffeq} remains invariant
\begin{equation}
\label{eq:spdiffeq2}
    \frac{d^2\Sigma'}{dt^2} = \frac{d\Sigma'}{dt}\Sigma'^{-1}\frac{d\Sigma'}{dt}
\end{equation}
However, boundary conditions are now diagonal, namely
\begin{align}\label{eq:boundry-fs}
    \Sigma'(t_0)&=D\\
    \Sigma'(t_1)&=\mbox{I}
\end{align}
therefore equation \ref{eq:spdiffeq2} reduces to a set of independent scalar equations, one equation for each term in the diagonal of $\Sigma'$.
The solution is equal to 
\begin{equation}
\label{eq:sigmaprime}
    \Sigma'(t)=D^{(t_1-t)/(t_1-t_0)}
\end{equation}
Note that the exponent $(t_1-t)/(t_1-t_0)$ measures precisely the relative distance travelled along the path.
In fact, since by definition $dt = \mathcal{L}ds$, then $t_1-t_0$ measures the total length $\mathcal{D}$
\begin{equation}
    t_1-t_0 = \int_{s_0}^{s_1} ds\;\mathcal{L}(\Sigma(s), \dot{\Sigma}(s))=\mathcal{D}
\end{equation}
We do not compute the value of $\mathcal{D}$, instead we rescale $t$ and use it to measure the \emph{relative} distance travelled along the path.
Therefore, we rewrite \ref{eq:sigmaprime} as  $\Sigma'(t)=D^{1-t}$ and the solution for $\Sigma$ is equal to
\begin{align}
    \Sigma(t)&=FD^{1-t}F^\dagger=\Sigma_1^{1/2}UD^{1-t}U^\dagger\Sigma_1^{1/2}=\\
    &=\Sigma_1^{1/2}\left(\Sigma_1^{-1/2}\Sigma_0\Sigma_1^{-1/2}\right)^{1-t}\Sigma_1^{1/2}
\end{align}

\end{proof}

\section{Forward process}
\label{app:forward process}

\begin{theorem}
Given a random vector $\mathbf{x}_0$ of zero mean and covariance $\Sigma_0$, and another random vector ${\boldsymbol \epsilon}_t$ of zero mean and covariance $\mbox{\emph{I}}$ (isotropic), where $\mathbf{x}_0$ and ${\boldsymbol \epsilon}_t$ are uncorrelated.
Assume the matrix $(\mbox{I}-\Sigma_0)$ is invertible.
Define the matrix $\Phi_t = (\mbox{I} -\Sigma_0^{1-t})(\mbox{I}-\Sigma_0)^{-1}$ and note that, for $t\in(0,1)$, $\Phi_t$ and $(\mbox{\emph{I}}-\Phi_t)$ are positive definite and, respectively, monotonically decreasing and increasing functions of $\Sigma_0$.
Then, the corrupted vector $\mathbf{x}_t$ defined by
\begin{equation}
\label{eq:corrupt}
    \mathbf{x}_t = \Phi_t^{\frac{1}{2}} \mathbf{x}_0 + (\mbox{I}-\Phi_t)^{\frac{1}{2}} {\boldsymbol \epsilon}_t
\end{equation}
has zero mean and covariance equal to $\Sigma_0^{1-t}$.

\end{theorem}

\begin{proof}

The mean of $\mathbf{x}_t$ is zero, because the mean of both $\mathbf{x}_0$ and ${\boldsymbol \epsilon}_t$ are zero.
The covariance of $\mathbf{x}_t$ can be calculated from equation \ref{eq:corrupt} by computing $\mathbf{x}_t\mathbf{x}_t^\dagger$ and averaging.
Note that $\mathbf{x}_0$ and ${\boldsymbol \epsilon}_t$ are uncorrelated.
Then, the covariance is equal to
\begin{equation}\label{eq:forward-var} 
    \Sigma_t = \Phi_t^{\frac{1}{2}} \Sigma_0 \Phi_t^{\frac{1}{2}} + (\mbox{I}-\Phi_t)
\end{equation}
Using the expression of $\Phi_t = (\mbox{I} -\Sigma_0^{1-t})(\mbox{I}-\Sigma_0)^{-1}$, we note that $\Phi_t$ and $\Sigma_0$ commute, therefore the covariance can be rewritten as
\begin{equation}\label{eq:forward-var-2} 
    \Sigma_t = \Phi_t \Sigma_0 + (\mbox{I}-\Phi_t)=\mbox{I}-\Phi_t(\mbox{I}-\Sigma_0)=\Sigma_0^{1-t}
\end{equation}

\end{proof}

\section{Time complexity of power spectrum}
\label{app:ps-time}

The time complexity of DFT and fit of the power spectrum are, respectively, quasilinear ($O(d\log(d))$) and linear ($O(d)$) in the number of pixels $d$. In our experiments, it took a few minutes to obtain the constants $c_1$ and $c_2$ for CIFAR. We also computed the power spectrum of ImageNet 64x64 on a CPU in about one hour. In both cases, the time required to calculate these hyperparameters is much less than the time required for training. Because the complexity of the power spectrum is not worse than the complexity of training, we expect the computation time of the former to be always much smaller even for datasets with higher resolution. 

Concerning the dataset size $n$, the complexity of DFT scales linearly with $n$, while the complexity of the fit of the power spectrum does not depend on the dataset size. Therefore, we do not expect this to be the limiting factor for our method. Furthermore, for much larger datasets, an estimate of the power spectrum may be obtained from a subset of the data.

\section{Generated images}
\label{app:gen images}

\begin{figure*}[th]
    \centering
    \includegraphics[width=\linewidth]{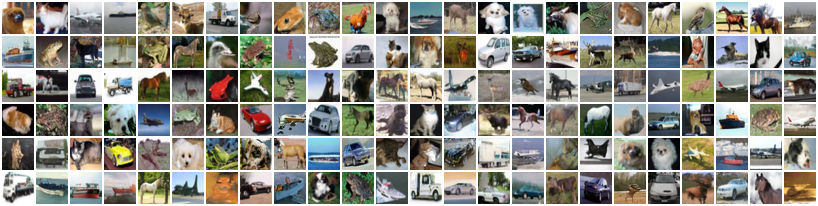}
    \vskip 0.1in
    \includegraphics[width=\linewidth]{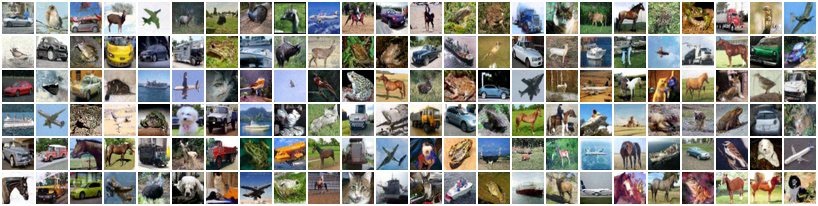}
    \vskip 0.1in
    \includegraphics[width=\linewidth]{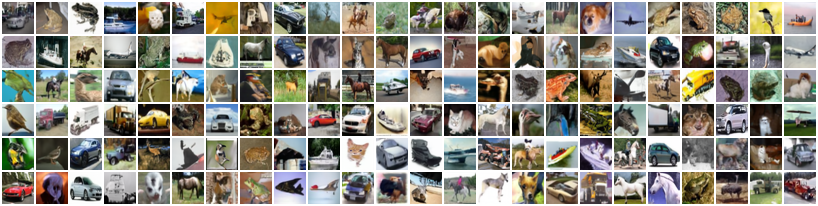}
    \caption{\textbf{Example generated images at T=500 diffusion steps for models trained on CIFAR10}. Top to bottom: SPD (FID = $2.74$), iDDPM \cite{iddpm} (FID = $4.20$) and iDDPM+DDIM \cite{song2020denoising} (FID = $3.74$).}
    \label{fig:qual_results}
\end{figure*}

\begin{figure}[t]
    \centering
    \includegraphics[width=\linewidth]{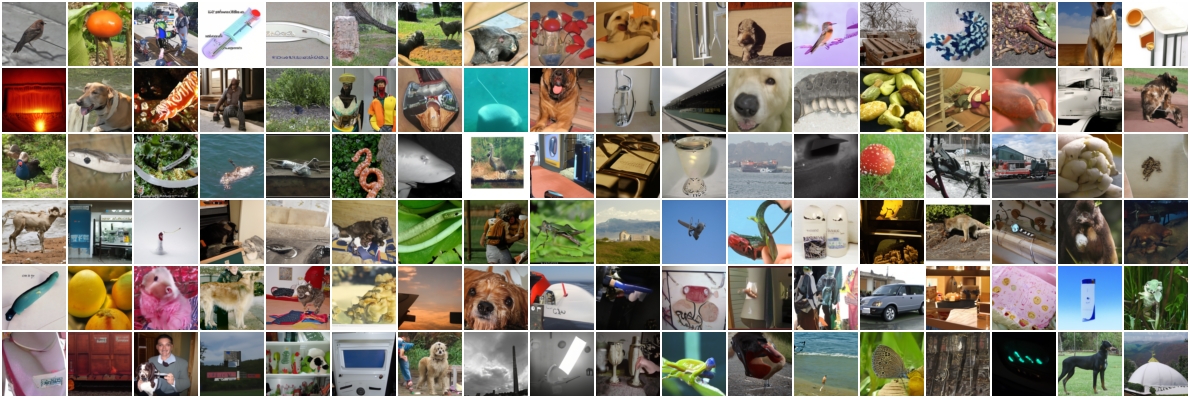}
    \caption{\textbf{Example generated images for model trained on ImagNet 64x64}. The model is trained at T=1000 diffusion steps and the evaluated FID is $13.7$.}
    \label{fig:IN64images}
\end{figure}

\end{document}